\documentclass{article}

\usepackage{microtype}
\usepackage{graphicx}
\usepackage{subcaption}
\usepackage{booktabs} % for professional tables

\usepackage{hyperref}

\usepackage[ruled, vlined,algo2e,linesnumbered]{algorithm2e}

\usepackage[preprint]{icml2026}

\usepackage{amsmath}
\usepackage{amssymb}
\usepackage{mathtools}
\usepackage{amsthm}

\usepackage[capitalize,noabbrev]{cleveref}

\theoremstyle{plain}
\newtheorem{theorem}{Theorem}[section]
\newtheorem{proposition}[theorem]{Proposition}
\newtheorem{lemma}[theorem]{Lemma}

\theoremstyle{definition}
\newtheorem{definition}[theorem]{Definition}
\newtheorem{assumption}[theorem]{Assumption}
\theoremstyle{remark}
\newtheorem{remark}[theorem]{Remark}

\usepackage[textsize=tiny]{todonotes}

\icmltitlerunning{Revealing Combinatorial Reasoning of GNNs via Graph Concept Bottleneck Layer}

\begin{document}

\twocolumn[
  \icmltitle{Revealing Combinatorial Reasoning of GNNs via Graph Concept Bottleneck Layer}




  \begin{icmlauthorlist}
    \icmlauthor{Yue Niu}{yyy}
    \icmlauthor{Zhaokai Sun}{yyy}
    \icmlauthor{Jiayi Yang}{yyy}
    \icmlauthor{Xiaofeng Cao}{yyy}
    \icmlauthor{Rui Fan}{yyy}
    \icmlauthor{Xin Sun}{sch}
    \icmlauthor{Hanli Wang}{yyy}
    \icmlauthor{Wei Ye}{yyy}
  \end{icmlauthorlist}

  \icmlaffiliation{yyy}{Tongji University, Shanghai, China}
  \icmlaffiliation{sch}{City University of Macau, Macau, China}

  \icmlcorrespondingauthor{Wei Ye}{yew@tongji.edu.cn}

  \icmlkeywords{Machine Learning, ICML}

  \vskip 0.3in
]



\printAffiliationsAndNotice{}  % no special notice (required even if empty)

\begin{abstract}
% Graph Neural Networks mark a fundamental shift in machine learning by enabling the modeling of complex relational structures beyond Euclidean data, achieving state-of-the-art performance in various domains. 
Despite their success in various domains, the growing dependence on GNNs raises a critical concern about the nature of the combinatorial reasoning underlying their predictions, which is often hidden within their black-box architectures. Addressing this challenge requires understanding how GNNs translate topological patterns into logical rules. However, current works only uncover the hard logical rules over graph concepts, which cannot quantify the contribution of each concept to prediction. Moreover, they are post-hoc interpretable methods that generate explanations after model training and may not accurately reflect the true combinatorial reasoning of GNNs, since they approximate it with a surrogate. In this work, we develop a graph concept bottleneck layer that can be integrated into any GNN architectures to guide them to predict the selected discriminative global graph concepts. The predicted concept scores are further projected to class labels by a sparse linear layer. It enforces the combinatorial reasoning of GNNs’ predictions to fit the soft logical rule over graph concepts and thus can quantify the contribution of each concept. To further improve the quality of the concept bottleneck, we treat concepts as ``graph words'' and graphs as ``graph sentences'', and leverage language models to learn graph concept embeddings. Extensive experiments on multiple datasets show that our method GCBMs achieve state-of-the-art performance both in classification and interpretability.

\end{abstract}

\section{Introduction}

Graph neural networks (GNNs)~\cite{hamilton2017inductive,kipf2017semi,velivckovic2017graph,zhang2018end,xu2018powerful} are prevalent in the learning tasks on graph data. However, they are inherently black-box models. The ``combinatorial reasoning'' they employ to arrive at predictions is opaque to humans. This limits their application in high-stakes fields such as drug discovery. A prediction without a logical ``why'' is often unplausible and unusable. Humans cannot verify if the GNN used valid chemical properties or just exploited spurious correlations in the dataset for prediction.

Combinatorial reasoning in GNNs describes the internal decision process through which a model discovers, evaluates, and composes discrete structural patterns—such as subgraphs or motifs—into a unified global decision rule that generalizes across a collection of graphs. Since GNNs explicitly contend with the discrete and combinatorial structure inherent to graph data, combinatorial reasoning goes beyond simple statistical correlation; it reflects a learned structural logic that can, at least in principle, be articulated as a Boolean formula over graph concepts (subgraphs or motifs).

Most interpretable methods~\cite{ying2019gnnexplainer,yuan2021explainability,xie2022task,yugraph,yu2022improving,zhang2022protgnn,seo2023interpretable,rao2025incorporating} are unable to reveal the ``combinatorial reasoning'' in GNNs. Recently, GraphTrail~\cite{armgaan2024graphtrail} and GLGExplainer~\cite{azzolinglobal} have been proposed to express the combinatorial reasoning of a trained GNN $\Phi$ as a set of $C$ (the number of graph classes) Boolean formulas $\{f_1,\ldots,f_C\}$ of graph concepts such that for any graph $G$, if $\Phi(G)=\mathbf{y}_i$ (the one-hot label of the $i$-th class), then $f_i(\mathcal{C})=True$ and $\forall j\neq i$, $f_j(\mathcal{C})=False$, where $\mathcal{C}$ contains all unique graph concepts in the dataset. It is obvious that the Boolean formulas are human-interpretable, but they are hard logical rules (i.e., they cannot quantify the importance of each concept that contributes to the prediction) and just proxies (post-hoc interpretation) for a GNN's prediction.

To tackle these limitations, we aim to develop a globally and intrinsically interpretable method that can identify the soft logical rules on the graph concepts, whose importance to the predictions is directly revealed. Our method is called Graph Concept Bottleneck Models (GCBMs), at the heart of which is a graph concept bottleneck layer that can be integrated into any GNNs and trained together. GCBMs first predict graph concepts, then map these predicted concepts to class labels, and identify the combinatorial reasoning (a soft logical rule) between concepts and class labels as a global interpretation.
Thus, they can quantify the contribution of each concept to predictions. Besides, it naturally provides an interface (graph concept bottleneck layer) that enables easy manual intervention to improve model accuracy while retaining interpretability. 

Building GCBMs entails three key challenges: graph concept extraction, discriminative concept selection, and high-quality graph concept label generation.
Specifically, we first adopt WL-subtrees~\cite{shervashidze2011weisfeiler} as graph concepts, as the message-passing mechanism of most GNNs follows the aggregate and update rule of the first-order Weisfeiler-Lehman (WL) graph isomorphism test~\cite{leman1968reduction}.
Second, we use information gain~\cite{kent1983information} to select discriminative concepts.
Finally, we obtain the concept scores by computing the similarity between the concept frequency vectors of the graphs and the one-hot encodings of the concepts. To improve the accuracy of the concept scores, we treat concepts as ``graph words'' and graphs as ``graph sentences'', leverage language models (LMs) to learn concept embeddings to capture their correlations and co-occurrences, and compute accurate concept scores based on these embeddings to optimize the quality of the concept bottleneck.

Our contributions can be summarized as follows:
\begin{itemize}
\item We propose Graph Concept Bottleneck Models (GCBMs), a new globally and intrinsically interpretable paradigm for GNNs that reveals the soft combinatorial reasoning of GNNs' predictions.
\item GCBMs are generic interpretation frameworks for GNNs. They can use any GNN as the backbone, any graph substructure as the graph concept, and any language model as the concept embedding method.
\item We verify GCBMs’ classification performance, interpretability, and intervenability on benchmark graph datasets across multiple domains, showing the advantages of GCBMs over baselines.
\end{itemize}

\section{Related Work}

\paragraph{Interpretable GNNs}
Current research on GNN interpretability is mainly categorized into post-hoc explanation and intrinsic interpretability. Post-hoc methods such as GNNExplainer~\cite{ying2019gnnexplainer} use independent models to trace the rationale behind black-box decisions. However, they suffer from unstable explanations and a tendency to capture spurious features that weaken model reliability~\cite{luo2020parameterized,vu2020pgm}.
Intrinsic methods embed explanatory logic into model design, such as GIB~\cite{yu2020graph} and ProtGNN~\cite{zhang2022protgnn}, which extract informative prototypes or subgraphs that support predictions but provide instance-level explanations and lack the ability to uncover combinatorial reasoning. GraphTrail~\cite{armgaan2024graphtrail} and GLGExplainer~\cite{azzolinglobal} use a surrogate to approximate the logical rules over graph concepts in GNNs, which may not uncover the truly inherent combinatorial reasoning. In addition, they are post-hoc methods and the logical rules uncovered by them are hard and cannot reveal the importance of each graph concept to prediction.

\paragraph{Concept Bottleneck Models} 
Concept bottleneck was first proposed in~\cite{koh2020concept} and used in image domain and now has been extended to text domain. In the image domain, methods such as Label-free CBM~\cite{oikarinen2023label}, LaBo~\cite{yang2023language}, and VLG-CBM~\cite{srivastava2024vlg} optimize the generation of concept sets and enhance their adaptability to different visual classification tasks.
Graph-CBMs~\cite{xu2025graph} treat image concepts as nodes to construct latent concept graphs, thus capturing the correlations between concepts. Graph-CBMs differ from the GCBMs proposed in this paper. GCBMs are designed for graph data.
In the text domain, frameworks such as TBMs~\cite{ludan2023interpretable,tan2024interpreting} and CB-LLM~\cite{sun2024concept} integrate CBMs with LLMs to meet the interpretability requirements of text-related tasks.

\section{Preliminaries}

An undirected and labeled graph is denoted by $G = (\mathcal{V}, \mathcal{E}, \mathbf{X})$, where $\mathcal{V}$ is a finite node set with features $\mathbf{X} = [\mathbf{x}_1, \ldots, \mathbf{x}_{|\mathcal{V}|}]^T\in\mathbb{R}^{|\mathcal{V}|\times d}$ ($\mathbf{x}_i \in \mathbb{R}^d$ is the feature vector of node $v_i$), and $\mathcal{E} \subseteq \mathcal{V} \times \mathcal{V} = \{(v_i, v_j) \mid v_i, v_j \in \mathcal{V}\}$ is the edge set, where an edge $e = (v_i, v_j)$ exists iff $e \in \mathcal{E}$. For the graph classification task, let $\mathcal{D} = \{G_i, \mathbf{y}_i\}_{i=1}^n$ denote the graph dataset, where $\mathbf{y}_i$ is the one-hot class label of graph $G_i$, and $n$ is the total number of graphs.

\begin{definition}[Graph Concept Bottleneck Layer]
Given a dataset $\mathcal{D} = \{G_i, \mathbf{y}_i\}_{i=1}^n$ consisting of $n$ input-target pairs, a graph concept bottleneck layer decomposes the classification process of GNNs into two stages: a backbone GNN $\Phi: G \rightarrow \mathbb{R}^M$, which projects $G$ into an $M$-dimensional concept space to generate the predicted concept label vector $\hat{\mathbf{c}} = \Phi(G) \in \mathbb{R}^M$; a linear classifier $f: \mathbb{R}^M \rightarrow \mathbb{R}^C$, which maps $\hat{\mathbf{c}}$ to the final prediction $\hat{\mathbf{y}} = f(\hat{\mathbf{c}}) \in \mathbb{R}^C$. 
\end{definition}

\begin{figure*}[!htb]
  % \vskip 0.2in
  \begin{center}
    \centerline{\includegraphics[width=2.0\columnwidth]{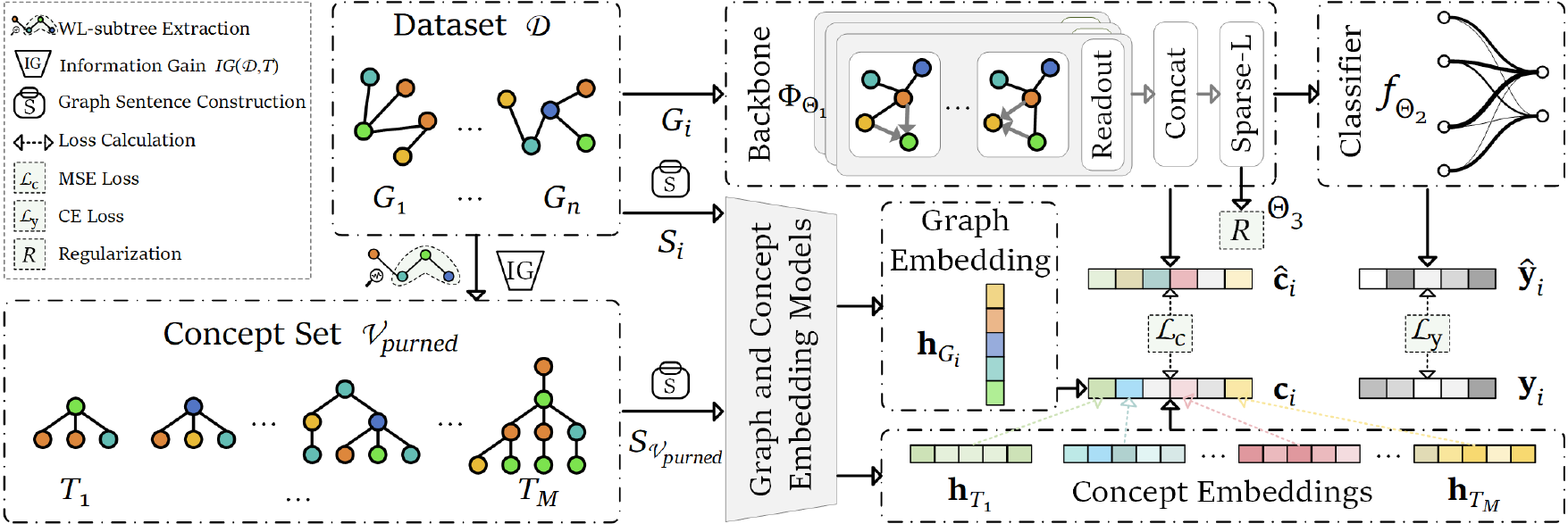}}
    \caption{
      Overview of the GCBMs construction pipeline.
    }
    \label{fig:gcbm pipeline}
  \end{center}
  \vskip -0.2in
\end{figure*}

Notably, the prediction $\hat{\mathbf{y}}$ is entirely dependent on the input $G$ through the bottleneck $\hat{\mathbf{c}}$, where each dimension of $\hat{\mathbf{c}}$ corresponds to a predefined concept. To build the graph concept bottleneck layer, a predefined concept set $\mathcal{T} = \{\mathbf{t}_m\}_{m=1}^M$ is necessary. Each concept $\mathbf{t}_m$ describes a topologically meaningful subgraph of the samples in $\mathcal{D}$.
For each input $G_i$, the ground-truth concept label $\mathbf{c}_i \in \mathbb{R}^M$ is computed by quantifying the relevance of $G_i$ to all concepts.
The dataset is then augmented with these concept labels as $\mathcal{D} = \{(G_i, \mathbf{c}_i, \mathbf{y}_i)\}_{i=1}^n$  for training both $\Phi$ and $f$.

We can directly observe the linear relationship between concepts and class labels through this layer, which is a soft logical rule that quantifies the contribution of each concept to the final prediction.

\section{Method: GCBMs}

This section presents the construction of the graph concept bottleneck layer, and the overall architecture of the GCBMs is illustrated in Figure~\ref{fig:gcbm pipeline}.

\subsection{Constructing Graph Concept Bottleneck Layer}

The message-passing mechanism in most GNNs is a continuous relaxation of the first-order WL graph isomorphism test~\cite{leman1968reduction}. Thus, we use WL-subtrees as graph concepts.
Figure~\ref{fig:wlsubtree} shows an undirected and labeled graph and a WL-subtree of height 2 rooted at the node with label 1. The WL graph isomorphism test iteratively compares the neighborhood of each node between two graphs, and as a byproduct, the encoding of the WL-subtree is generated. This WL-subtree encoding can be represented as ``1,(2,13)(3,123)'', which is an integration of the encodings of three WL-subtrees (``1,23'', ``2,13'', ``3,123'') of height 1 rooted at the nodes with labels 1, 2, 3, respectively.

\begin{figure}[!htb]
  % \vskip 0.2in
  \begin{center}
    \centerline{\includegraphics[width=0.6\columnwidth]{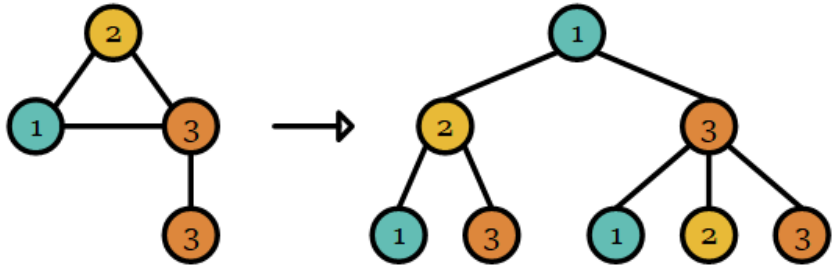}}
    \caption{
      On the left side is a graph, and on the right side is its WL-subtree of height 2 rooted at the node with label 1.
    }
    \label{fig:wlsubtree}
  \end{center}
  \vskip -0.2in
\end{figure}

For a graph dataset \( \mathcal{D} = \{ (G_i, \mathbf{y}_i) \}_{i=1}^n \), we first collect WL-subtrees of height \( k \) for all nodes in each graph and keep them in the graph concept corpus $\mathcal{C}^{(k)}$. 
Thus, $\mathcal{C}^{(k)}$ is a global-level corpus that spans all graphs in $\mathcal{D}$.
Subsequently, we extract non-isomorphic WL-subtrees from $\mathcal{C}^{(k)}$ to form the graph concept vocabulary $\mathcal{V}^{(k)}=\{T_1^{(k)},\ldots,T_{|\mathcal{V}^{(k)}|}^{(k)}\}$, where all the non-isomorphic graph concepts are sorted lexicographically by their encodings. 

Since not all graph concepts in $\mathcal{V}^{(k)}$ are relevant for predicting graph class labels, we need to identify an effective graph concept set $\mathcal{V}_{pruned}^{(k)}$ from $\mathcal{V}^{(k)}$ to construct the graph concept bottleneck layer for GNNs. Specifically, $\mathcal{V}_{pruned}^{(k)}$ is selected by measuring the mutual information of each concept to the graph class label. Thus, we use information gain $IG(\mathcal{D}, T^{(k)})$ to quantify the relevance of graph concept $T^{(k)}$ to class label. For each graph $G_i$, we compute its representation vector $\mathbf{x}_{G_i}^{(k)}$, where the $j$-th entry corresponds to the occurrence frequency of concept $T_j^{(k)}$ in $G_i$.
For the dataset $\mathcal{D}$, we then construct a matrix $\mathbf{X}^{(k)}$ with the $i$-th row $\mathbf{x}_{i,:}^{(k)}$ denoting the representation vector of the $i$-th graph and the $j$-th column $\mathbf{x}_{:,j}^{(k)}$ denoting the occurrence frequency of the $j$-th graph concept across all graphs.
The information gain for the $j$-th graph concept $T_j^{(k)}$ can be computed as:
\begin{equation}
\label{ig}
IG(\mathcal{D}, T_j^{(k)}) = IG(\mathcal{D}, \mathbf{x}_{:,j}^{(k)})=H(\mathcal{D}) - H(\mathcal{D}|\mathbf{x}_{:,j}^{(k)}),
\end{equation}
where $H(\mathcal{D})$ is the entropy of the class distribution in $\mathcal{D}$, and $H(\mathcal{D}|\mathbf{x}_{:,j}^{(k)})$ is the conditional entropy of the class distribution given graph concept $T_j^{(k)}$.

We select the top $M$ graph concepts with the highest information gain to form the graph concept set $\mathcal{V}_{pruned}^{(k)}$. 
Each graph concept $T_j^{(k)}\in \mathcal{V}_{pruned}^{(k)}$ ($1\leq j\leq M$) can be represented as a one-hot vector $\mathbf{t}_j^{(k)}$ of length $|\mathcal{V}^{(k)}|$ with the entry on its position in $\mathcal{V}^{(k)}$ as 1 and others 0.
Similar to CBMs, we define the ground-truth concept label/score of graph $G_i$ as the vector:
\begin{equation}
\label{eqn:concept_label}
\mathbf{c}_{i}^{(k)} = [E(\mathbf{x}_{G_i}^{(k)}, \mathbf{t}_1^{(k)}),\ldots,E(\mathbf{x}_{G_i}^{(k)}, \mathbf{t}_M^{(k)})],
\end{equation}
where $\mathbf{c}_{i}^{(k)}\in \mathbb{R}^M$ and its $j$-th score $E(\mathbf{x}_{G_i}^{(k)}, \mathbf{t}_j^{(k)})$ is computed as the cosine similarity between $\mathbf{x}_{G_i}^{(k)}$ and $\mathbf{t}_j^{(k)}$:
\begin{equation}\label{eqn:concept_label_}
E(\mathbf{x}_{G_i}^{(k)}, \mathbf{t}_j^{(k)}) = \frac{\langle \mathbf{x}_{G_i}^{(k)}, \mathbf{t}_j^{(k)} \rangle}{\|\mathbf{x}_{G_i}^{(k)}\| \cdot \|\mathbf{t}_j^{(k)}\|},
\end{equation}
where $\langle \cdot, \cdot \rangle$ denotes the dot product, and $\| \cdot \|$ denotes the norm. 

We concatenate all $\mathbf{c}_{i}^{(k)}$ ($1\leq k\leq K$) to obtain the global ground-truth graph concept label $\mathbf{c}_{i}=\text{Concat}(\mathbf{c}_{i}^{(1)}, \ldots, \mathbf{c}_{i}^{(K)})$. We adopt the $K$-layer graph isomorphism network (GIN)~\cite{xu2018powerful} as the GNN backbone $\Phi$, as it has been theoretically proven that its discriminative power is equivalent to the first-order WL graph isomorphism test. We obtain the global graph representation $\mathbf{h}_{G_i}$ via a concatenation of each output of the GIN layer. Then, we align the hierarchical outputs with the global ground-truth graph concept labels. This alignment ensures each layer’s output retains graph structural properties and can be interpreted as the extent to which the selected concepts exist in the graph.

To simultaneously optimize the concept learning and classification accuracy, GCBM's objective function comprises three components:
\begin{equation}
\mathcal{L} = \lambda_c \mathcal{L}_c + \mathcal{L}_y + \lambda_r R,
\end{equation}
where $\lambda_c, \lambda_r > 0$ denote trade-off parameters that balance concept learning, classification performance, and concept sparsity. The concept loss $\mathcal{L}_c = \frac{1}{|\mathcal{D}|} \sum_{G_i \in \mathcal{D}} \| \Phi_{\Theta_1}(G_i) - \mathbf{c}_{i} \|^2$ facilitates alignment between graph representations and their concept labels. The cross-entropy classification loss $\mathcal{L}_y = \frac{1}{|\mathcal{D}|} \sum_{G_i \in \mathcal{D}} CE\left(f_{\Theta_2}(\Phi_{\Theta_1}(G_i)), \mathbf{y}_{i}\right)$ ensures consistency between predicted and true class labels. The sparsity regularization term $R = \mathcal{L}_1(\Theta_3)$ enhances the sparsity of concept vectors and suppresses overfitting, where $\mathcal{L}_1(\cdot)$ denotes the $L_1$ regularization applied to the weight matrix $\Theta_3$ of the sparse linear layer ($\text{Sparse-L}$).

\subsection{Graph Concept Embedding}

In the above section, the ground-truth graph concept labels/scores computed based on the one-hot embeddings of graphs and concepts fail to capture correlations and co-occurrences between concepts and their probability distributions in and across graphs, leading to inaccurate ground-truth concept labels/scores.
In image and text domains, concepts within the concept set often exhibit inherent correlations. For instance, ``bone spurs'' and ``intervertebral space'' in images, or ``history'' and ``archaeology'' in text.
Similarly, graph concepts may share analogous semantic correlations and occurrence patterns. A typical example is the structural similarity between pyridine rings and pyrimidine rings in molecules, which cannot be reflected by their orthogonal one-hot representations.

To mitigate this issue, we analogize each graph concept as a “graph word” and draw inspiration from LMs. These models capture correlations and co-occurrences between tokens (``graph words'') through full-sequence context modeling. The analogy is supported by prior observations~\cite{perozzi2014deepwalk,yanardag2015deep} that graph substructure distributions follow a power law that is consistent with word distributions in natural languages.
As shown in Figure~\ref{fig:power_law}, WL-subtree distributions in graphs approximate the power law distribution.
Thus, LMs such as Transformer~\cite{vaswani2017attention} can effectively learn WL-subtree (graph concept) embeddings while capturing their correlations and co-occurrences.

\begin{figure}[!htb]
  \centering
  \includegraphics[width=0.37\textwidth]{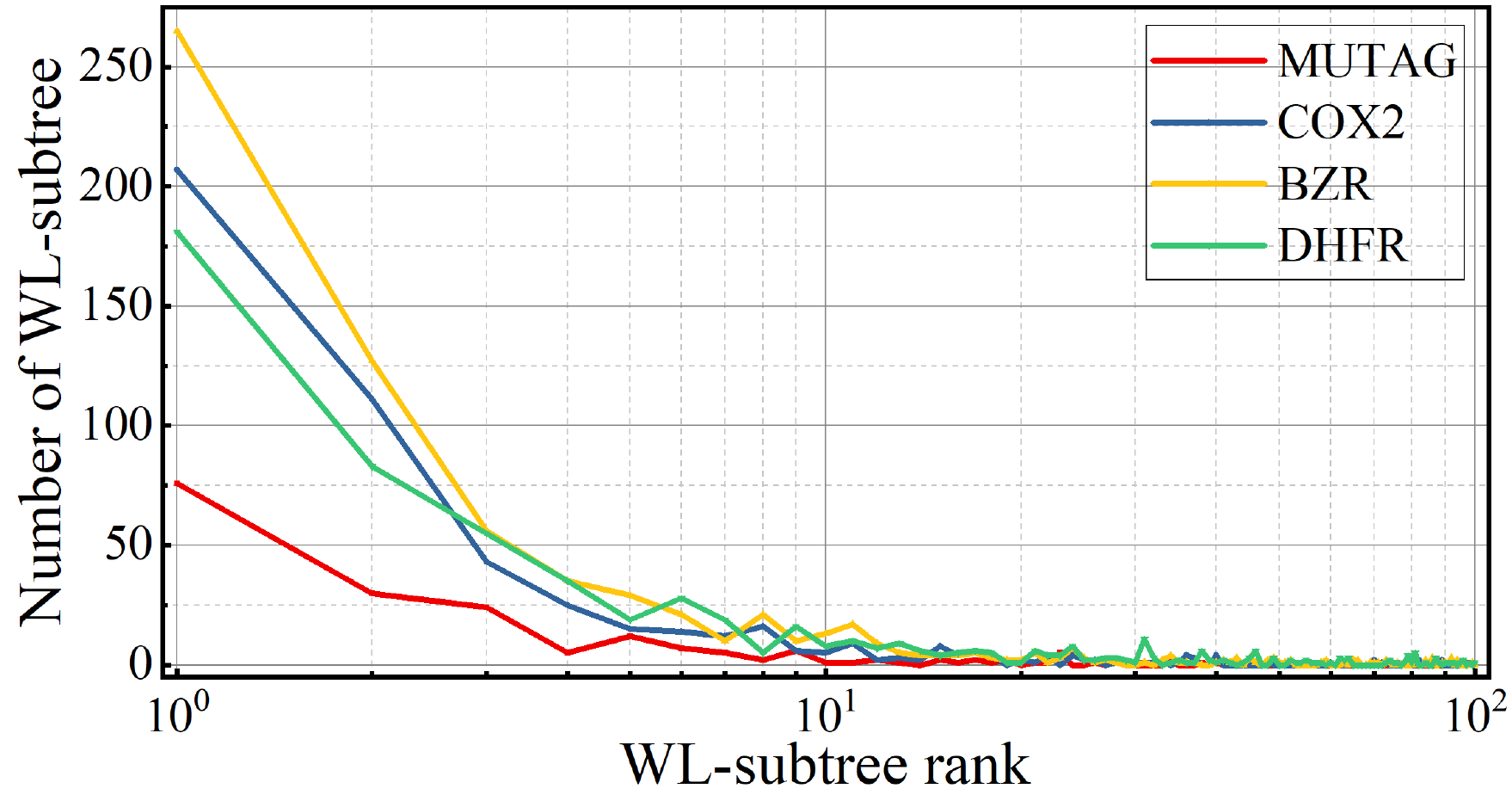}
  \caption{Distributions of the WL-subtrees ($\text{height} \leq 2$) in graph datasets MUTAG, COX2, BZR, and DHFR. The $x$-axis denotes WL-subtree rank (i.e., the number of occurrences of a WL-subtree), and the $y$-axis denotes the number of WL-subtrees with the same rank. These distributions approximate a power law.}
  \label{fig:power_law}
\end{figure}

Specifically, for graph $G_i$, we build its ``graph sentence'' $S_{i}^{\left(k\right)}=(T_{v_{\pi(1)}}^{(k)},\dots,T_{v_{\pi(|\mathcal{V}_{i}|)}}^{(k)})$, where $T_{v_{\pi(u)}}^{(k)} (1\leq u\leq |\mathcal{V}_{i}|)$ denotes the WL-subtree of height $k$ rooted at node $v_{\pi(u)}$ and $\pi(\cdot)$ is a permutation function on node indices in $G_i$, because the ``graph words'' in this ``graph sentence'' are sorted lexicographically by the encodings of their WL-subtrees. 
For concept $T_j^{(k)} (1\leq j\leq M)$, we also construct a ``graph sentence'' $S_{\mathcal{V}_{pruned}^{\left(k\right)}}=(T_1^{(k)},\dots,T_M^{(k)})$ based on $\mathcal{V}_{pruned}^{(k)}$, and the ``graph words'' are sorted in the same way as above.
Then, for graph dataset \( \mathcal{D} = \{ (G_i, \mathbf{y}_{i}) \}_{i=1}^n \), we add all ``graph sentences'' at levels $k=1,\ldots, K$ into a corpus $\mathcal{C}_{gs}$, and use the Transformer to learn dense embeddings of graph concepts:
\begin{align}
\mathbf{h}_{T_{v_{\pi(u)}}^{(k)}} &= \text{LM}(T_{v_{\pi(u)}}^{(k)}, \mathcal{C}_{gs}),\label{eq:ht}\\
\mathbf{h}_{T_j^{(k)}} &= \text{LM}(T_j^{(k)}, \mathcal{C}_{gs}),\label{eq:tj}
\end{align}
where $\mathbf{h}_{T_{v_{\pi(u)}}^{(k)}} \in \mathbb{R}^d$ denotes the embedding of the WL-subtree $T_{v_{\pi(u)}}^{(k)}$, which is also treated as the embedding of node $v_{\pi(u)}$. Similarly, $\mathbf{h}_{T_j^{(k)}} \in \mathbb{R}^d$ denotes the embedding of concept $T_j^{(k)}$. Notably, the embedding dimension $d$ is much smaller than the dimension of one-hot vectors ($|\mathcal{V}^{(k)}|$).

\begin{table*}[!htb]
  \caption{Classification accuracy (mean\%$\pm$standard deviation\%) of GCBMs and baseline models of 10-fold cross-validation on benchmark datasets. The best results are in \textbf{bold} and the second-best results are as \underline{underlined} (This notation applies to all subsequent tables).}
  \label{tab:accuracy}
  % \begin{center}
    % \begin{small}
    %   \begin{sc}
    \centering
        \resizebox{\linewidth}{!}{
        \begin{tabular}{llcccccccccc}
          \toprule
          &Dataset       &MUTAG &PTC-MR &BZR &COX2 &DHFR &PROTEINS &IMDB-B &IMDB-M &COLLAB &DBLP\_v1\\
          \midrule
          Group 1 &GNNExplainer  &66.5$\pm$2.3 &67.1$\pm$6.2 &81.1$\pm$10.4   &80.7$\pm$6.6 &69.8$\pm$7.2  &71.9$\pm$4.9 &70.4$\pm$4.2 &45.0$\pm$4.3 &77.6$\pm$5.0 &82.8$\pm$0.4 \\
                  &PGExplainer   &78.6$\pm$11.4 &67.0$\pm$5.5 &82.9$\pm$2.7   &72.2$\pm$2.7 &69.0$\pm$6.4  &72.5$\pm$3.5 &62.2$\pm$4.0 &40.0$\pm$6.5 &75.7$\pm$4.5 &81.2$\pm$0.5 \\
                  &ProtGNN       &77.6$\pm$4.7 &68.2$\pm$5.8 &82.5$\pm$2.5    &79.2$\pm$3.1 &70.8$\pm$2.3  &74.3$\pm$2.6 &58.3$\pm$7.3 &36.0$\pm$5.7 &69.3$\pm$8.6 &80.0$\pm$0.6 \\
                  &VGIB          &72.1$\pm$21.6 &66.5$\pm$6.1 &79.8$\pm$3.9   &80.0$\pm$4.2 &51.7$\pm$10.5 &68.4$\pm$6.5 &50.8$\pm$3.7 &33.9$\pm$3.1 &72.3$\pm$4.8 &80.5$\pm$0.8 \\
                  &PGIB          &76.6$\pm$6.8 &66.9$\pm$5.4 &83.7$\pm$5.8    &75.8$\pm$2.3 &77.4$\pm$4.1  &74.7$\pm$3.1 &57.0$\pm$2.9 &37.0$\pm$3.9 &73.2$\pm$4.9 &81.3$\pm$0.7 \\
                  %& GIP           & 83.5$\pm$5.1 &67.3$\pm$5.8  &80.9$\pm$7.0    &79.4$\pm$2.8 &74.5$\pm$10.9  &79.5$\pm$3.5 &66.5$\pm$4.5 &41.9$\pm$4.3 &78.2$\pm$4.9  &83.1$\pm$0.5  \\
                  &ConfExplainer &82.5$\pm$6.7 &66.8$\pm$6.1  &83.1$\pm$5.2   &75.5$\pm$3.8 &71.8$\pm$8.2  &72.5$\pm$5.2 &65.6$\pm$5.4 &44.9$\pm$4.5 &77.5$\pm$5.2 &82.7$\pm$0.6 \\
                  &GraphTrail    &87.4$\pm$5.0 &69.3$\pm$5.9  &86.7$\pm$4.7   &83.3$\pm$4.9 &77.2$\pm$3.6  &73.5$\pm$7.9 &72.9$\pm$4.9 &46.5$\pm$2.1 &67.3$\pm$4.2 &79.2$\pm$0.5   \\
          \midrule
          Group 2 &WL &87.4$\pm$5.4   &56.0$\pm$3.9 &81.3$\pm$0.6 &81.2$\pm$1.1 &82.4$\pm$0.9 &74.4$\pm$2.6  &67.5$\pm$4.0  &48.4$\pm$4.2  &78.5$\pm$1.7  &77.6$\pm$0.6  \\
                  &WWL &86.3$\pm$7.9   &52.6$\pm$6.8 &87.6$\pm$0.6 & 82.7$\pm$3.0 & 82.7$\pm$3.6 & 73.1$\pm$1.4  & 71.6$\pm$3.8  & 52.6$\pm$3.0    & 81.4$\pm$2.1  & 81.8$\pm$0.6  \\
                  & FWL & 85.7$\pm$7.5   & 59.3$\pm$6.5 & 87.6$\pm$5.0 & 82.0$\pm$2.7 & 82.5$\pm$4.0 & 74.6$\pm$3.8  & 72.0$\pm$4.6  & 49.5$\pm$2.6  & 78.1$\pm$1.2  & 81.9$\pm$0.8  \\
                  & GAWL & 87.3$\pm$6.3   & 61.6$\pm$4.6 & 88.3$\pm$3.8 & 79.9$\pm$4.9 & 83.1$\pm$4.0 & 74.7$\pm$3.0  & 74.6$\pm$4.7  & 50.7$\pm$5.4  & 81.5$\pm$2.0  & 81.1$\pm$0.3  \\
          \midrule
          Group 3 & OT-GNN & 91.6$\pm$4.6   & 68.0$\pm$7.5 & 85.9$\pm$3.3 & 83.5$\pm$3.6 & 84.8$\pm$3.1 & 76.6$\pm$4.0  & 67.5$\pm$3.5  & 52.1$\pm$3.0  & 80.7$\pm$2.9  & 85.4$\pm$0.3  \\
                  & WEGL & 91.0$\pm$3.4   & 66.2$\pm$6.9 & 84.4$\pm$4.6 & 81.8$\pm$1.8 & 82.2$\pm$2.9 & 73.7$\pm$1.9  & 66.4$\pm$2.1  & 50.3$\pm$1.0   & 79.6$\pm$0.5  & 80.4$\pm$0.4  \\
          \midrule
          Group 4 & GIN & 89.4$\pm$5.6   & 64.6$\pm$7.0 & 82.6$\pm$3.5 & 81.6$\pm$4.9 & 82.6$\pm$2.6 & 76.2$\pm$2.8  & 64.3$\pm$3.1  & 50.9$\pm$1.7  & 79.3$\pm$1.7  & 83.3$\pm$0.4  \\
                  & DROPGIN & 90.4$\pm$7.0   & 66.3$\pm$8.6 & 77.8$\pm$2.6 & 82.1$\pm$2.7 & 82.7$\pm$3.3 & 76.9$\pm$4.3  & 66.3$\pm$4.5  & 51.6$\pm$3.2  & 80.1$\pm$2.8  & 83.0$\pm$0.9  \\
                  & SEP & 89.4$\pm$6.1   & 68.5$\pm$5.2 & 86.9$\pm$0.8 & 83.7$\pm$1.8 & 84.1$\pm$1.2 & 76.4$\pm$0.4  & 74.1$\pm$0.6  & 51.5$\pm$0.7  & 81.3$\pm$0.2  & 85.5$\pm$0.3  \\
                  & GMT & 89.9$\pm$4.2   & 70.2$\pm$6.2 & 85.6$\pm$0.8 & 84.3$\pm$2.0 & 83.9$\pm$0.9 & 75.1$\pm$0.6  & 73.5$\pm$0.8  & 50.7$\pm$0.8  & 80.7$\pm$0.5  & 85.2$\pm$0.6  \\
                  & MINCUTPOOL & 90.6$\pm$4.6   & 68.3$\pm$4.4 & 87.2$\pm$1.0 & 83.5$\pm$2.7 & 82.2$\pm$0.9 & 74.7$\pm$0.5  & 72.7$\pm$0.8  & 51.0$\pm$0.7  & 80.9$\pm$0.3  & 84.9$\pm$0.5  \\
                  & ASAP & 87.4$\pm$5.7   & 64.6$\pm$6.8 & 85.3$\pm$1.3 & 81.1$\pm$3.1 & 83.3$\pm$1.6 & 73.9$\pm$0.6  & 72.8$\pm$0.5  & 50.8$\pm$0.8  & 78.6$\pm$0.5  & 82.6$\pm$0.6  \\
                  & WITTOPOPOOL & 89.4$\pm$5.4   & 64.6$\pm$6.8 & 87.8$\pm$2.4 & 83.7$\pm$3.6 & 84.5$\pm$2.8 & \underline{80.0$\pm$3.2}  & 72.6$\pm$1.8  & \underline{52.9$\pm$0.8}  & 80.1$\pm$1.6  & 85.6$\pm$0.8  \\
                  & GRDL & 92.1$\pm$5.9   & 68.3$\pm$5.4 & \textbf{92.0$\pm$1.1} & \textbf{85.9$\pm$2.2} & 85.1$\pm$2.3 & \textbf{82.6$\pm$1.2}  & \underline{74.8$\pm$2.0}  & 52.9$\pm$1.8  & 79.8$\pm$0.9  & \underline{86.2$\pm$0.5}  \\
          \midrule
          Group 5 & GRAPHORMER & 89.6$\pm$6.2   & \textbf{71.4$\pm$5.2} & 85.3$\pm$2.3 & 84.0$\pm$2.9 & 83.7$\pm$2.1 & 76.3$\pm$2.7  & 70.3$\pm$0.9  & 48.9$\pm$2.0  & 80.3$\pm$1.3  & 84.8$\pm$0.6  \\
                   & SAT & 92.6$\pm$4.3   & 68.3$\pm$4.9 & \underline{91.7$\pm$2.1} & 84.1$\pm$2.6 & 84.4$\pm$3.4 & 77.7$\pm$3.2  & 70.0$\pm$1.3  & 47.3$\pm$3.2  & 80.6$\pm$0.6  & 85.9$\pm$0.3  \\
          \midrule
          Ours   & GCBM  & \underline{93.5$\pm$4.0} & 68.5$\pm$5.8 & 89.7$\pm$4.2 & 84.3$\pm$3.5 & \underline{85.7$\pm$3.6} & 75.3$\pm$2.1 & 74.2$\pm$4.1 & 52.5$\pm$3.1 & \underline{82.4$\pm$1.6} & 86.2$\pm$0.7 \\ 
                  & GCBM-E  & \textbf{93.6$\pm$4.0} & \underline{70.6$\pm$5.4} & 90.4$\pm$2.8 & \underline{85.9$\pm$3.2} & \textbf{86.1$\pm$3.2} & 76.8$\pm$2.7  & \textbf{75.4$\pm$4.3} & \textbf{53.1$\pm$2.4} & \textbf{82.8$\pm$1.4} & \textbf{87.6$\pm$0.5} \\
          \bottomrule
        \end{tabular}
        }
  %     \end{sc}
  %   \end{small}
  % \end{center}
  % \vskip -0.1in
\end{table*}

The representation $\mathbf{h}_{G_i}^{(k)}$ of $G_i$ is obtained by aggregating all embeddings of WL-subtrees of height $k$ rooted at each node in $G_i$:
\begin{equation}\label{eq:hg}
\mathbf{h}_{G_i}^{(k)} = \sum_{T_{v_{\pi(u)}}^{(k)} \in S_{G_i}^{(k)}} \mathbf{h}_{T_{v_{\pi(u)}}^{(k)}}.
\end{equation} 
\vskip -0.1in

The graph concept label $\mathbf{c}_{i}^{(k)}$ in Equation (\ref{eqn:concept_label}) is updated by replacing $\mathbf{x}_{G_i}^{(k)}$ with the graph embedding $\mathbf{h}_{G_i}^{(k)}$ and $\mathbf{t}_{j}^{(k)}$ by the graph concept embedding $\mathbf{h}_{T_j^{(k)}}$.
The context-aware embeddings capture correlations and co-occurrence information between concepts, and these dense embeddings can alleviate the high-dimensional sparsity of one-hot vectors for nodes and concepts. The refined concept labels better align with the intrinsic graph structure, providing a more effective concept bottleneck for GNNs. The generation process of embedding-based graph concept labels is detailed in Algorithm~\ref{alg:gcbm_lm} (see Appendix~\ref{Appendix_alg}).

\section{Experimental Evaluation}

In this section, we validate the performance of GCBMs in graph classification and test the interpretability of the concept bottleneck layer via human evaluation and quantitative assessment. We further analyze the contributions of the core components via ablation studies, conduct parameter sensitivity analysis, and perform intervention experiments.\footnote{Code available at \url{https://anonymous.4open.science/r/gcbms-332C131/}}

\subsection{Experimental Setup}

We conduct classification experiments on 10 public datasets and 2 large-scale class-imbalanced datasets, and perform interpretability experiments on 2 specialized interpretation datasets.
To comprehensively evaluate GCBM’s performance, we select 5 groups of representative advanced methods as baselines: interpretable GNN models (Group 1), WL-based graph kernels (Group 2), optimal transport-based models (Group 3), mainstream GNN models (Group 4), and graph Transformer models (Group 5). Detailed statistics of datasets and baselines, as well as the parameter settings and hardware configuration, are provided in Appendix~\ref{Appendix_data}.

\subsection{Experimental Results and Analysis}

We refer to the GCBM constructed with concept bottlenecks derived from one-hot representations and dense embeddings (generated via a Transformer) as GCBM and GCBM-E, respectively.

\subsubsection{Graph Classification Performance} 

Table~\ref{tab:accuracy} reports the classification accuracy of the two GCBM versions on 10 datasets of various domains, while Table~\ref{tab:auc_roc_large_datasets} presents GCBM-E's AUC scores~\cite{wang2024graph,luo2020parameterized} on 2 class-imbalanced datasets.

Table~\ref{tab:accuracy} shows that on 10 small, medium, and large datasets, GCBM-E  ranks among the top 2 on 8 datasets, and achieves the highest accuracy on 6 of them. 
Specifically, it nearly outperforms all baselines in Groups 1--3 across datasets, outperforms most Group 4 baselines (only slightly lagging behind the best baselines on BZR and PROTEINS), and remains leading over Group 5 baselines on MUTAG, IMDB-B, IMDB-M, and COLLAB (performing comparably on PTC-MR and BZR).
Compared with GCBM-E, GCBM uses simple one-hot vectors as graph concept embeddings to generate concept scores. Its average accuracy decreases by 1.0\%. This is because one-hot representations fail to capture contextual differences of the same WL-subtree across different graphs, and ignore correlations and co-occurrences between concepts, leading to degraded quality of concept scores. Additionally, their high sparsity hinders effective graph embedding learning.
Table~\ref{tab:auc_roc_large_datasets} shows that on 2 large-scale class-imbalanced graphs, GCBM-E  achieves AUC scores of 0.888 and 0.862, significantly outperforming the best baseline methods.

\begin{table}[!htb]
  \caption{AUC scores (mean\%$\pm$standard deviation\%) of 10-fold cross-validation of each method on class-imbalanced benchmark datasets (only results of methods that can complete the experiments in 24 hours are reported).}
  \label{tab:auc_roc_large_datasets}
  \begin{center}
    \begin{small}
      % \begin{sc}
        \setlength{\tabcolsep}{4pt}
        \begin{tabular}{lcc}
          \toprule
          Dataset   & PC-3 & MCF-7 \\
          \midrule
          % PGIB    &- &- \\
          GNNExplainer & 84.1$\pm$1.6 & 79.7$\pm$1.5 \\
          ConfExplainer & 83.7$\pm$1.5 & 79.2$\pm$1.4     \\
          GIN     & 84.6$\pm$1.4   & 80.6$\pm$1.5   \\
          GRDL    & \underline{85.1$\pm$1.6}   & \underline{81.4$\pm$1.3}   \\
          \midrule
          GCBM-E  & \textbf{88.8$\pm$1.7} & \textbf{86.2$\pm$1.3} \\
          \bottomrule
        \end{tabular}
      % \end{sc}
    \end{small}
  \end{center}
  \vskip -0.2in
\end{table}

Detailed analysis of running time and computational complexity is presented in Appendix~\ref{appendix:time_cost}.
In summary, GCBMs achieve excellent performance on multiple datasets, especially in class-imbalanced scenarios, for which we further provide theoretical proof of their superiority in AUC scores (see Appendix~\ref{Appendix_auc}).

\subsubsection{Interpretability Evaluation}

To verify the interpretability of GCBMs, we conduct human evaluation and quantitative assessment.

\paragraph{Human Evaluation}
We perform human evaluation from two perspectives. First, we visualize the classification layer weights by plotting a sankey diagram~\cite{oikarinen2023label} for the DHFR dataset in Figure~\ref{fig:sankey_diagram}.
To obtain the specific functional groups corresponding to the concepts in the sankey diagram, we map the nodes of the associated WL-subtrees (graph concepts) back to their original molecules and extract the relevant atoms and bonds, as shown in Figure~\ref{fig:functional groups}.
%=============================================================
\begin{figure}[!htb]
  \begin{center}
    \centerline{\includegraphics[width=0.85\columnwidth]{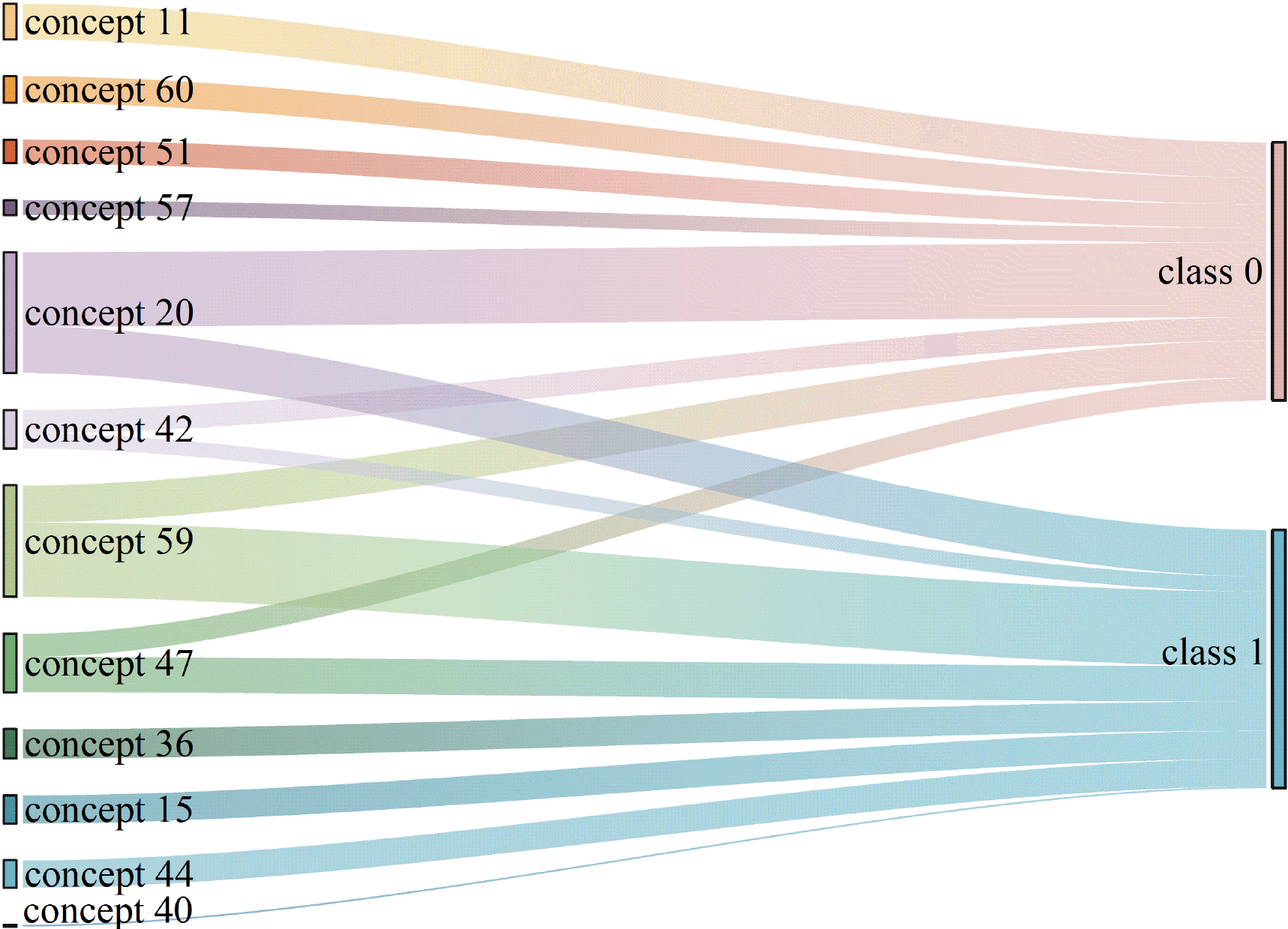}}
    \caption{
      The sankey diagram of the classification layer weights of GCBM on the DHFR dataset.  
      The width of the connection between each concept and class corresponds to the exponent of the absolute value of the learned concept weight. To highlight core concepts, only the top 8 concepts with the largest absolute values of weights in each class are shown.
    }
    \label{fig:sankey_diagram}
  \end{center}
    \vskip -0.2in
\end{figure}
%=============================================================
\begin{figure}[!htb]
  \begin{center}
    \centerline{\includegraphics[width=0.8\columnwidth]{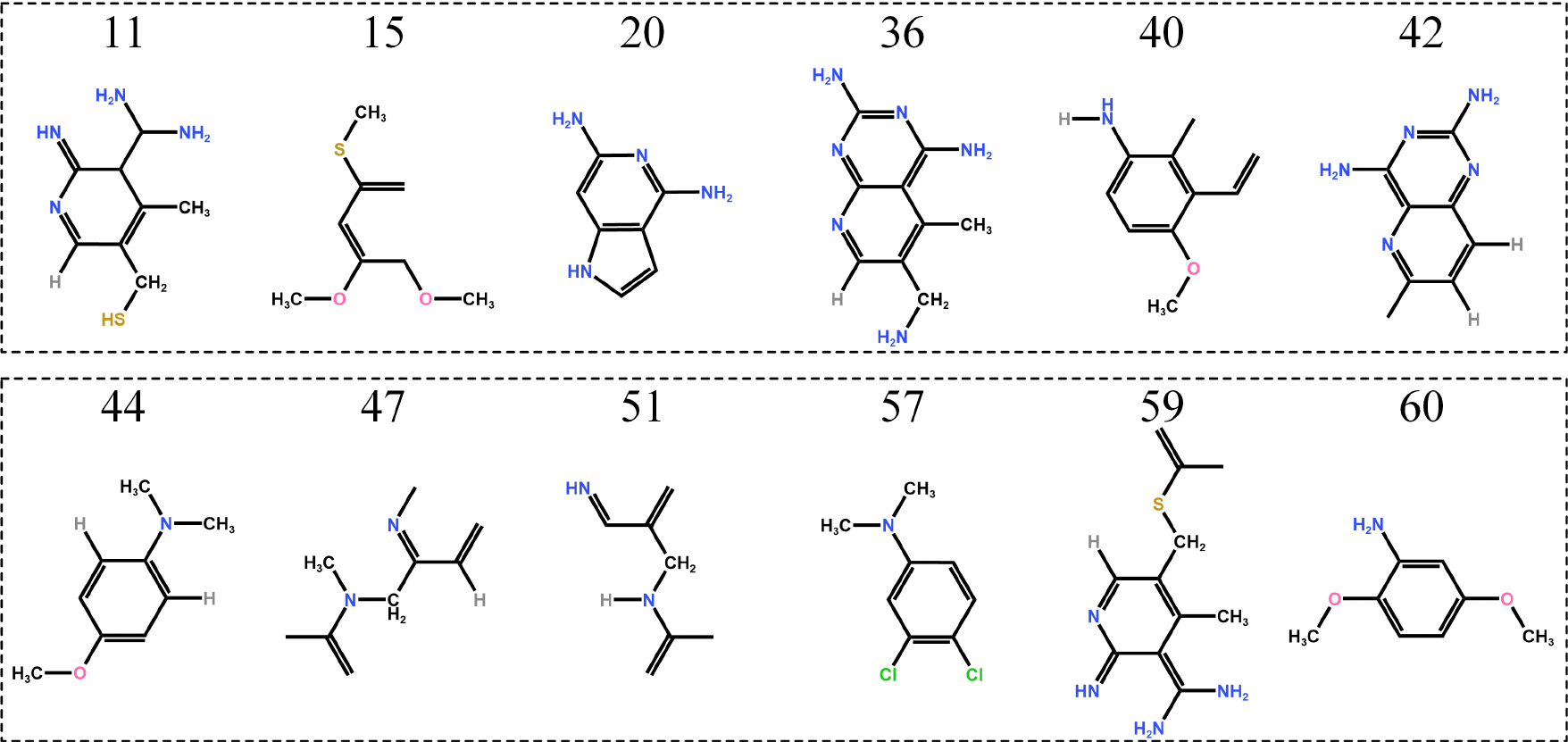}}
    \caption{
      Concept indices and their corresponding specific functional groups in molecules from the DHFR dataset.
    }
    \label{fig:functional groups}
  \end{center}
  \vskip -0.2in
\end{figure}
%=============================================================

As observed in Figure~\ref{fig:sankey_diagram}, graph concepts 20, 42, 47, and 59 have high weights across both classes, indicating their critical role in molecule classification. Molecules from different classes present distinct concept score distributions for these cross-class concepts. When making decisions, GCBMs prioritize the scores on these concepts. This visualization provides a clear global explanation that reveals how GCBMs make decisions through the linear combinatorial reasoning of these global concepts.

Second, we verify that these discriminative concepts correspond to human-understandable semantic information.
Specifically, we restore the concepts to their original graph structures and validate their consistency with domain knowledge.
As shown in Figure~\ref{fig:molecular samples}, we present the graph structures of sample molecules from two classes that contain concepts 20 and 42.
%=============================================================
As observed in Figure~\ref{fig:molecular samples}, concept 20 corresponds to a purine ring structure
and concept 42 to a pteridine ring structure. 
Both substructures have been confirmed by biochemical studies as key active groups of dihydrofolate reductase inhibitors~\cite{srinivasan2015insights,tassone2021evidence,possart2023silico,hoarau2023discovery}, and are highly relevant to the target task labels.
These results confirm that GCBMs can identify meaningful concepts consistent with real-world domain knowledge through concept selection, verifying the interpretability of the graph concept bottleneck.
\begin{figure}[!htb]
  \begin{center}
    \centerline{\includegraphics[width=0.7\columnwidth]{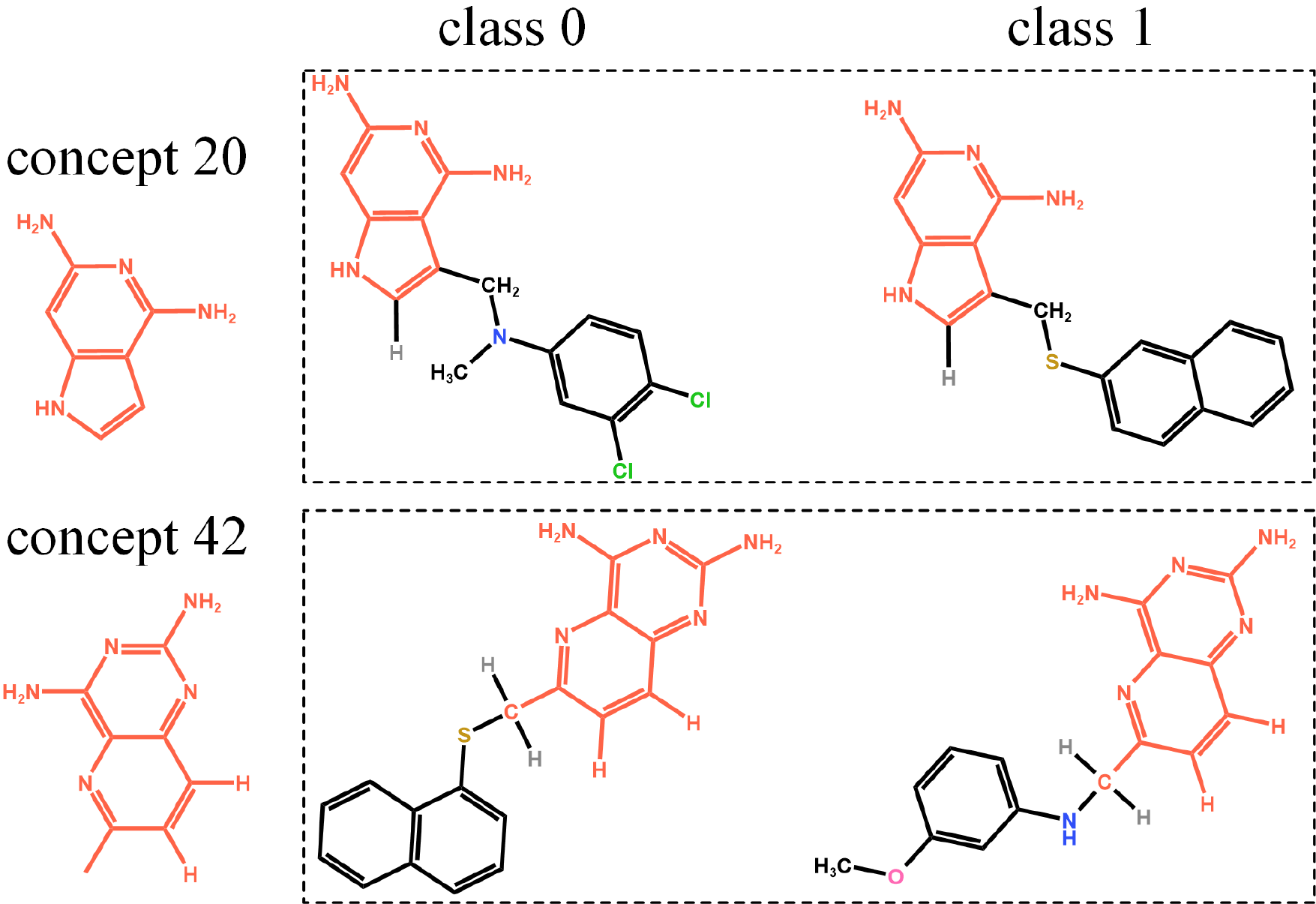}}
    \caption{
      Key graph concepts and sample molecules with these concepts from each class in the DHFR dataset. The red part indicates the target substructure that each key concept is mapped back to in the original graph.
    }
    \label{fig:molecular samples}
  \end{center}
    \vskip -0.2in
\end{figure}
%=============================================================

\paragraph{Quantitative Assessment}
We adopt the AUC score as the interpretability metric, which quantifies the alignment degree between the true and predicted interpretable subgraphs.
Specifically, the quantitative assessment is conducted on 2 datasets Solubility and Benzene with annotations of the true interpretable subgraphs.
Using the GCBM's concept selection method, we filter the top $M_I=2$ key concepts from the graph concept vocabulary $\mathcal{V}^{(k)}$, map these concepts back to each original graph to derive their specific substructures, and finally take the union of nodes from all substructures as GCBM's predicted interpretable subgraphs for each graph.
Table~\ref{tab:Interpretation performance} reports GCBM’s AUC scores on these 2 datasets.

The results show that GCBM achieves higher AUC scores, confirming that the concepts used to construct the explicit concept layer can accurately correspond to the interpretable substructures in graphs.
Notably, unlike baselines, GCBM's subgraph prediction avoids complex training processes. 
Its key concept extraction and subgraph mapping are achieved solely by adjusting the hyperparameter $M_I$, greatly reducing the computational cost.
The parameter sensitivity analysis of $M_I$ is provided in Appendix~\ref{Appendix_M_I}.

%=============================================================
\begin{table}[!htb]
  \caption{AUC scores (mean\%) of each method on datasets Solubility and Benzene for interpretability evaluation. The selected two concepts are shared for GCBM and GCBM-E, leading to the same results.}
  \label{tab:Interpretation performance}
  \begin{center}
    \begin{small}
      % \begin{sc}
        \setlength{\tabcolsep}{4pt}  
        \begin{tabular}{lccccc}
          \toprule
          Dataset       &Solubility   &Benzene   \\
          \midrule
          GNNExplainer  & 51.8   & 54.7    \\
          PGExplainer   & 77.9   & 48.2    \\
          %SubgraphX    & 83.4   & 77.1    \\
          ProtGNN       & 56.2   & 66.4    \\
          VGIB          & 67.5   & 62.2    \\
          PGIB          & 65.0   & 61.3    \\
          ConfExplainer & \underline{80.5}   & \underline{75.3}    \\
          \midrule
          GCBM/GCBM-E         & \textbf{91.2}   & \textbf{83.6}    \\
          \bottomrule
        \end{tabular}
      % \end{sc}
    \end{small}
  \end{center}
  \vskip -0.2in
\end{table}
%=============================================================

\subsubsection{Ablation Studies}

To verify the importance of GCBM-E’s key components, we construct 3 more variants by removing each component individually and compare their performance on 4 benchmark datasets.
Table~\ref{tab:gcbm_variants_accuracy} shows that GCBM-E achieves the best performance across all datasets.
GCBM-$\mathcal{V}$ replaces the highly discriminative concept set $\mathcal{V}_{pruned}^{(k)}$ with $M$ randomly selected concepts. Its accuracy is reduced across all datasets, as many selected concepts are ineffective: they have low information gain and contribute little to class discrimination, leading to similar concept scores for graphs of different classes. Consequently, the classification layer struggles to distinguish graphs based on these scores.
GCBM-GCN and GCBM-G replace the backbone GIN with GCN~\cite{kipf2017semi} and the graph concept embedding model Transformer with GloVe~\cite{pennington2014glove}, respectively. The performance gap between these variants and GCBM-E mainly stems from differences in the representation learning capabilities of the substituted models. This result suggests that GCBMs can achieve stronger performance with more powerful backbones and graph concept embedding models.
In summary, all key components of GCBM-E contribute significantly to performance. 
%=============================================================
\begin{table}[!htb]
  \caption{Classification accuracy (mean\%$\pm$standard deviation\%) of GCBM-E and its variants on benchmark datasets. GCBM-$\mathcal{V}$ adopts a randomly selected concept set; GCBM-GCN replaces the backbone GIN with GCN; GCBM-G replaces Transformer with GloVe.}
  \label{tab:gcbm_variants_accuracy}
  \begin{center}
    % \begin{small}
      % \begin{sc}
      %   \setlength{\tabcolsep}{4pt}
      \resizebox{\linewidth}{!}{
        \begin{tabular}{lcccc}
          \toprule
          Dataset       & MUTAG      & COX2      & DHFR      & IMDB-M    \\
          \midrule
          GCBM-$\mathcal{V}$      & 92.5$\pm$4.3   & 85.0$\pm$2.5  & 85.7$\pm$3.3  & \underline{53.0$\pm$3.3}  \\
          GCBM-GCN      & 87.7$\pm$6.2   & 85.0$\pm$3.4  & 84.2$\pm$2.4  & 52.2$\pm$3.4  \\
          GCBM-G & \underline{93.0$\pm$4.2} & \underline{85.7$\pm$2.9} & \underline{86.0$\pm$3.7} &52.4$\pm$2.1 \\
          \midrule
          GCBM-E  & \textbf{93.6$\pm$4.0} & \textbf{85.9$\pm$3.2} & \textbf{86.1$\pm$3.2} & \textbf{53.1$\pm$2.4} \\
          \bottomrule
         \end{tabular}
         }
      % \end{sc}
    % \end{small}
  \end{center}
  \vskip -0.2in
\end{table}
%=============================================================

\subsubsection{Sensitivity Analysis}

In this section, we evaluate the sensitivity of GCBM-E's performance to its parameters $\lambda_c$ (graph concept loss weight), $K$ (the maximum height of the WL-subtrees), and $\lambda_r$ (sparse regularization weight) on the same 4 datasets as the above section.
%=============================================================
\begin{figure*}[!htb]
  % \vskip 0.2in
  \begin{center}
    \centerline{\includegraphics[width=2.05\columnwidth]{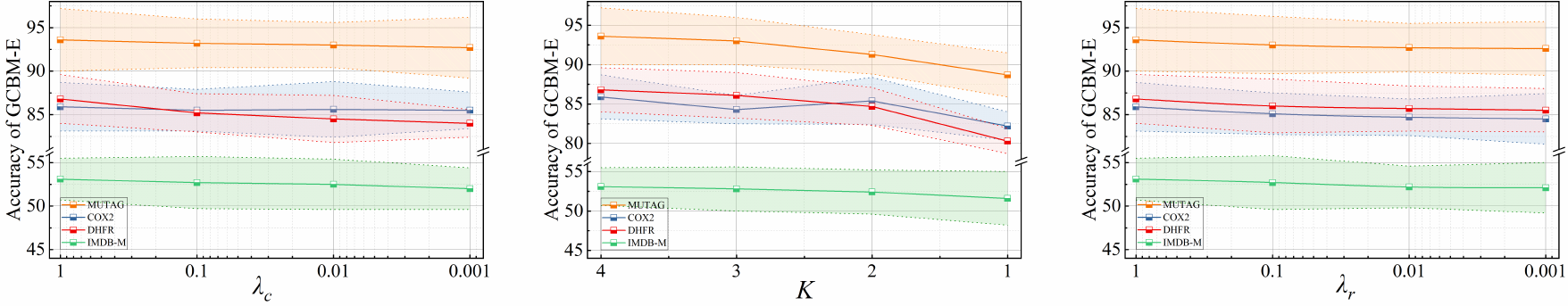}}
    \caption{
      Parameter sensitivity analysis of GCBM-E: $\lambda_c$ (graph concept loss weight), $K$ (the maximum height of the WL-subtrees), and $\lambda_r$ (sparse regularization weight).
    }
    \label{fig:sensitivity}
  \end{center}
\end{figure*}
%=============================================================
\begin{table*}[t]
  \caption{Intervention results of GCBM and GCBM-E on the DHFR dataset. $W_{F[\text{cls}, \text{cpt}]}$ denotes the weight of the target concept (cpt) corresponding to the class (cls) in the final classification layer weight matrix $W_F$. NMR indicates no modification required. Accuracy represents the classification accuracy (mean\%) before and after intervention.}
  \label{tab:intervention}
  \centering
  % \begin{center}
    % \begin{small}
    %   \begin{sc}
    %     \setlength{\tabcolsep}{4pt}  
        \resizebox{\linewidth}{!}{
        \begin{tabular}{lccccc}
          \toprule
          Model       &Intervention Target                          & $W_{F[\text{cls 0}, \text{cpt j}]}$          & $W_{F[\text{cls 1}, \text{cpt j}]}$          &Impact                                      &Accuracy       \\
          \midrule
          GCBM        & $\hat{\mathbf{c}}_{i}$ of misclassified samples      & NMR                                           & NMR                                           & +12 corrections                               & 85.7 $\to$ 87.3 \\
          GCBM-E       & $W_{F[\text{cls 0}, \text{cpt 20}]}$, $W_{F[\text{cls 1}, \text{cpt 20}]}$ & 0.384 $\to$ -0.037                       & -0.011 $\to$ 0.410                       & +11 corrections; +6 new errors                & 86.1 $\to$ 86.8 \\
          GCBM-E       & $W_{F[\text{cls 0}, \text{cpt 42}]}$, $W_{F[\text{cls 1}, \text{cpt 42}]}$ & 0.188 $\to$ -0.247                        & 0.094 $\to$ 0.528                        & +8 corrections; +5 new errors                 & 86.1 $\to$ 86.5 \\
          GCBM-E       & $W_{F[\text{cls 0}, \text{cpt 20/42}]}$, $W_{F[\text{cls 1}, \text{cpt 20/42}]}$ & 0.384/0.188 $\to$ -0.037/-0.247 & -0.011/0.094 $\to$ 0.410/0.528 & +15 corrections; +8 new errors                & 86.1 $\to$ 87.0 \\
          \bottomrule
        \end{tabular}
        }
    %   \end{sc}
    % \end{small}
  % \end{center}
  % \vskip -0.1in
\end{table*}
%=============================================================
Figure~\ref{fig:sensitivity} demonstrates that adjusting GCBM-E’s key parameters results in performance degradation.
Specifically, decreasing $\lambda_c$ or $\lambda_r$ reduces the model’s accuracy, while decreasing $K$ leads to a more significant performance drop.
Compared to weakening the alignment between graph representations and concept labels (reducing $\lambda_c$) or relaxing the constraints of sparse regularization (reducing $\lambda_r$), pruning the concept vocabulary (decreasing $K$) fundamentally undermines the quality of the concept bottleneck, thus causing more severe performance impairment.
Overall, reducing $\lambda_c$, $K$, or $\lambda_r$ consistently degrades GCBM-E’s performance across all datasets.

\subsubsection{Intervention Experiments}

To verify the intervenability of the concept layer in GCBMs, we adopt different intervention strategies for GCBM and GCBM-E, and conduct experiments on their misclassified graphs in the DHFR dataset.
Specifically, GCBM’s concept score for graph $G_i$ on concept $T_j$ relies on $G_i$'s frequency-based representation vector and the one-hot vector of $T_j$. For misclassified graphs, we directly edit their predicted concept vectors $\hat{\mathbf{c}}_{i}$~\cite{koh2020concept} and feed the modified concept vectors into the classification layer to re-predict the class, thus correcting the misclassified graphs.
For GCBM-E, concept scores are computed via the similarity of Transformer-learned embeddings, and the associations between concepts and classes are measured by the classification layer weights. For graph $G_i$ satisfying $\hat{\mathbf{y}}_{i} \neq \mathbf{y}_{i}$ and $\text{sim}(\hat{\mathbf{c}}_{i}, \mathbf{c}_{i}) \geq {\tau}_c$ (i.e., valid concept activation but incorrect classification~\cite{oikarinen2023label}), where ${\tau}_c$ denotes the similarity threshold for valid concept activation, we design three intervention schemes: modifying the classification layer weights of discriminative concept 20 alone, modifying those of concept 42 alone, and performing joint intervention on both concepts explore the synergistic effect of key concepts. Detailed modification rules are provided in Appendix~\ref{Appendix_intervention}, and intervention results are shown in Table~\ref{tab:intervention}.

The results show that GCBM’s classification accuracy increases from 85.7\% to 87.3\%, with the correction of misclassified graphs achieving a net improvement.
For GCBM-E, all three intervention schemes enhance overall performance, and the joint intervention exhibits the most significant effect: it corrects 15 misclassified graphs (outperforming the sum of individual interventions on each concept) with only 8 new misclassifications, boosting accuracy from 86.1\% to 87.0\%.
This superior performance validates the combinatorial reasoning of GCBMs. Key concepts do not act independently but synergistically contribute to the model’s decision-making.
These intervention experiments confirm GCBMs’ advantages in interpretability and decision controllability, enabling targeted optimization of prediction results through intervention on core concept combinations.

\section{Conclusion}

This paper proposes a globally and intrinsically interpretable paradigm for GNNs called Graph Concept Bottleneck Models (GCBMs).
As a generic framework, the explicit graph concept bottleneck layer of GCBMs enforces the combinatorial reasoning of GNNs' predictions to fit the soft logical rule over graph concepts, achieving both excellent discriminative power and interpretability.
Experiments on multiple datasets verify the superior predictive performance of GCBMs, with notable advantages on large-scale class-imbalanced datasets. 
For interpretability, human evaluation of the molecular graphs confirms that the concepts extracted by GCBMs are highly aligned with domain knowledge, while quantitative assessment demonstrates that these concepts accurately correspond to interpretable subgraphs in graphs.
Moreover, intervention experiments validate the controllability of GCBMs’ decision-making process.
In the future, we plan to expand graph concept types (e.g., textual descriptions of graph properties) using large language models and explore more advanced backbone GNNs (e.g., graph Transformers) to further enhance the applicability of GCBMs on complex graph data.

%================================================

\section*{Impact Statement}
This paper focuses on advancing technological innovation and development in the field of Machine Learning. While our research may yield multifaceted societal impacts, no specifc aspects require special emphasis following careful evaluation.

\nocite{langley00}

\bibliography{example_paper}
\bibliographystyle{icml2026}

% APPENDIX

\newpage
\appendix
\onecolumn  %%%%%%%%%%%  one-column  or  two-column

\section{The Generation of Embedding-based Graph Concept Labels}
\label{Appendix_alg}

To clarify our method, we present the generation process of the embedding-based graph concept labels in Algorithm~\ref{alg:gcbm_lm}.

\begin{algorithm2e}[!htb]
  \caption{Generation of Embedding-based Graph Concept Labels}
  \label{alg:gcbm_lm}
  \KwIn{Graph dataset $\mathcal{D} = \{ (G_i, \mathbf{y}_i) \}_{i=1}^n$, Max WL-subtree height $K$, Number of selected concepts $M$, Pre-trained language model $LM$}
  \KwOut{Complemented dataset $\mathcal{D}=\{(G_i, \mathbf{c}_i, \mathbf{y}_i)\}_{i=1}^n$}
  \tcp{\textbf{Phase 1: Build graph sentence corpus}}
  Initialize the graph sentence corpus $\mathcal{C}_{gs} = \emptyset$\;
  \For{$k = 1$ \KwTo $K$}{
    Initialize the $k$-level graph concept corpus $\mathcal{C}^{(k)} = \emptyset$\;
    Extract the WL-subtree of height $k$ for each node $v$ in each graph $G_i$ and add it to $\mathcal{C}^{(k)}$\;
    Form the graph concept vocabulary $\mathcal{V}^{(k)}=\{T_1^{(k)},\ldots,T_{|\mathcal{V}^{(k)}|}^{(k)}\}$ with the unique WL-subtrees in $\mathcal{C}^{(k)}$\;
    Build the frequency matrix $\mathbf{X}^{(k)} \in \mathbb{R}^{n \times|\mathcal{V}^{(k)}|}$ of $\mathcal{D}$, where $\mathbf{X}^{(k)}[i][j]= \text{Frequency}(G_i,T_j^{(k)})$\;
    Calculate information gain for each concept $IG(\mathcal{D},T_j^{(k)})$ via Eq.~(\ref{ig})\;
    Select top-$M$ concepts from $\mathcal{V}^{(k)}$ to form the graph concept set $\mathcal{V}_{pruned}^{(k)}= \{T_1^{(k)},\ldots,T_{M}^{(k)}\}$\;
    Build ``graph sentences'' $S_i^{(k)}$ for each $G_i$ and $S_{\mathcal{V}_{pruned}^{(k)}}$ for $\mathcal{V}_{pruned}^{(k)}$ with sorted WL-subtree codes\;
    Form the $k$-level graph sentence corpus $\mathcal{C}_{gs}^{(k)}$ with $S_i^{(k)}$ and $S_{\mathcal{V}_{pruned}^{(k)}}$\;
    $\mathcal{C}_{gs} \leftarrow \mathcal{C}_{gs} \cup \mathcal{C}_{gs}^{(k)}$\;
  }
  \tcp{\textbf{Phase 2: Generate graph concept label}}
  Train $LM$ on $\mathcal{C}_{gs}$\;
  \For{$k = 1$ \KwTo $K$}{
    Generate node embeddings $\mathbf{h}_{T_{v_{\pi(u)}}^{(k)}}$ and concept embeddings $\mathbf{h}_{T_j^{(k)}}$ via Eq.~(\ref{eq:ht}) and Eq.~(\ref{eq:tj})\;
    Aggregate the node embeddings of $G_i$ and generate graph embeddings $\mathbf{h}_{G_i}^{(k)}$ via Eq.~(\ref{eq:hg})\;
    Calculate concept scores via Eq.~(\ref{eqn:concept_label_}) and generate $k$-level concept labels $\mathbf{c}_i^{(k)}$ via Eq.~(\ref{eqn:concept_label})\;
  }
  Generate global graph labels $\mathbf{c}_i=\text{Concat}(\mathbf{c}_i^{(1)}, \ldots, \mathbf{c}_i^{(K)})$\;
  Complement the dataset $\mathcal{D}=\{(G_i, \mathbf{c}_i, \mathbf{y}_i)\}_{i=1}^n$ with $\mathbf{c}_i$\;
  \Return{$\mathcal{D}$}
\end{algorithm2e}

\section{Experimental Setup}
\label{Appendix_data}
\paragraph{Datasets}
We conducted experiments on 10 public graph datasets (covering bioinformatics: MUTAG, PTC-MR, BZR, COX2, DHFR, PROTEINS; social networks: IMDB-B, IMDB-M, COLLAB, DBLP\_v1) for graph classification tasks across small, medium, and large scales,
along with 2 large-scale class-imbalanced datasets PC-3 and MCF-7 to test performance in imbalanced scenarios, and 2 specialized interpretation datasets Solubility~\cite{duvenaud2015convolutional}, Benzene~\cite{sanchez2020evaluating} for interpretability evaluation. Dataset statistics are available in~\cite{Morris+2020}.

\paragraph{Baselines} 
We select 5 groups of representative advanced methods as baselines for a comprehensive evaluation, categorized as follows.

Interpretable GNN models (Group 1): These methods focus on identifying graph substructures that support the model’s correct predictions, including GNNExplainer~\cite{ying2019gnnexplainer}, PGExplainer~\cite{luo2020parameterized}, ProtGNN~\cite{zhang2022protgnn}, VGIB~\cite{yu2022improving}, PGIB~\cite{seo2023interpretable}, ConfExplainer~\cite{zhang2025your}, GraphTrail~\cite{armgaan2024graphtrail}.

WL-based graph kernels (Group 2): These methods measure graph similarity via hand-crafted kernel functions, serving as classical baselines in graph classification, including WL-subtree kernel~\cite{shervashidze2011weisfeiler}, WWL~\cite{togninalli2019wasserstein}, FWL~\cite{schulz2022graph}, and GAWL~\cite{giannis2023graph}.

Optimal transport-based models (Group 3): These methods align graph distribution features using optimal transport theory and perform prominently in similarity measurement tasks, including OT-GNN~\cite{chen2020optimal} and WEGL~\cite{kolouri2020wasserstein}.

Mainstream GNN models (Group 4): These methods capture graph structural features through hierarchical aggregation, representing the mainstream of graph representation learning. These models include GIN~\cite{xu2018powerful}, DropGIN~\cite{papp2021dropgnn}, SEP~\cite{wu2022structural}, GMT~\cite{baek2021accurate}, MinCutPool~\cite{bianchi2020spectral}, ASAP~\cite{ranjan2020asap}, WITTOPOPOOL~\cite{chen2023topological}, and GRDL~\cite{wang2024graph}.

Graph Transformer models (Group 5): These methods introduce self-attention mechanisms to model long-range dependencies in graphs, demonstrating potential in handling complex graph structures. These models include GRAPHORMER~\cite{ying2021transformers} and SAT~\cite{chen2022structure}.

\paragraph{Parameter Settings and Hardware Configuration}
The default parameters for each comparison method are used. For GCBMs:

We adopt the Transformer variant BERT~\cite{devlin2018bert} as the LM. Both the number of hidden layers and that of attention heads are set to 4, and the hidden dimension is set to 64. The maximum length of text sequences is set to the length of the longest sequence among the ``graph sentence'' in the dataset. Input tokens are randomly masked at a 15\% rate for self-supervised training. The BERT model is trained for 3 epochs with cross-entropy as the loss function, and the batch size is adjusted according to the dataset size. 

GCBM's backbone GIN consists of 4 layers with a fixed hidden dimension of 64.
The readout operation follows the default setting of GIN, using global sum pooling. To enhance the propagation of deep representations, we introduce residual connections for each layer with learnable residual weights. In addition, we select $M=64$ highly discriminative concepts for each layer and dynamically adjust the importance of concepts at different levels via learnable hierarchical concept weights. Furthermore, a dropout mechanism is introduced in each layer to prevent overfitting. 

During training, concept labels are incorporated to compute concept loss; during testing, no concept labels are required, and the model automatically predicts concept labels and outputs classification results. The total number of training epochs is set to 300, with the Adam~\cite{kingma2017adam} optimizer employed. The learning rate scheduling strategy is ReduceLROnPlateau, with a minimum improvement threshold of 0.1 and a minimum learning rate of $10^{-6}$. We use the Optuna~\cite{optuna_2019} optimization framework to search for hyperparameters, including dropout rate, initial learning rate, weight decay, learning rate decay factor, patience, and cooldown period. Additionally, a gradient clipping strategy is applied to stabilize training and prevent gradient explosion. Experiments perform 10-fold cross-validation.

Experiments are conducted on a server equipped with an Intel (R) Xeon (R) Gold 6226R processor, 256GB memory, an NVIDIA RTX 3090 GPU, and Ubuntu 18.04.6 LTS as the operating system. The deep learning framework used is PyTorch 1.13.0.

\section{Runing Time and Complexity Analysis}  
\label{appendix:time_cost}
Table~\ref{tab:preprocessing time} reports the preprocessing time of GCBM-E  on different datasets, including three components: WL-subtree extraction, information gain computation, and Transformer training. Table~\ref{tab:preprocessing time} shows that the preprocessing time of GCBM-E  is minimal on small datasets and remains acceptable on large datasets.
%===============================================
\begin{table}[!htb]
  \caption{Preprocessing time (in seconds) of each component in GCBM-E .}
  \label{tab:preprocessing time}
  \begin{center}
    \begin{small}
      % \begin{sc}
        \setlength{\tabcolsep}{3pt}
        
        \begin{tabular}{lccccc}
          \toprule
          Component                      & MUTAG    & COX2    & DHFR       & IMDB-M    & COLLAB     \\
          \midrule
          WL-subtree extraction          & 4.43    & 4.97    & 6.49       & 7.54      & 16840.92    \\
          Information gain computation    & 0.68    & 1.20    & 1.24       & 1.19      & 2435.85    \\
          Transformer training           & 7.27    & 10.53   & 15.59     & 22.11    & 39720.67    \\
          \bottomrule
        \end{tabular}
        
      % \end{sc}
    \end{small}
  \end{center}
  % \vskip -0.1in
\end{table}

Table~\ref{tab:training time} reports the training time of GCBM, GCBM-E and several representative baseline models.
As shown in Table~\ref{tab:training time}, WWL (a WL-based graph kernel method) achieves the shortest training time, as graph kernel methods do not require constructing or training neural networks. Among GNN-based methods, GIN is implemented via TensorFlow, while all other compared models are developed with PyTorch. For fair comparison, all GNN-based methods use a unified setup of 350 training epochs and a batch size of 64. These results confirm that GCBMs yield superior training efficiency, achieving a good balance between computational cost and predictive performance.

%=================================================
\begin{table}[!htb]
  \caption{Training time (in seconds) of GCBMs and representative baselines.}
  \label{tab:training time}
  \begin{center}
    \begin{small}
      % \begin{sc}
        \setlength{\tabcolsep}{4pt}
        
        \begin{tabular}{lccccc}
          \toprule
          Dataset       & MUTAG    & COX2     & DHFR     & IMDB-M    & COLLAB     \\
          \midrule
          GNNExplainer & 388.22   & 677.19   & 771.87   & 2058.63   & 122602.41  \\
          WWL          & 4.75     & 44.52    & 117.29   & 263.62    & 10044.34   \\
          GIN          & 392.35   & 726.21   & 981.19   & 2452.42   & 126747.07  \\
          GRDL         & 195.76   & 356.20   & 541.52   & 1438.94   & 107040.14  \\
          \midrule
          GCBM         & 177.60   & 337.38   & 476.58   & 1353.61   & 96191.98  \\
          GCBM-E        & 198.93   & 388.93   & 577.29   & 1521.60   & 106338.83  \\
          \bottomrule
        \end{tabular}
        
      % \end{sc}
    \end{small}
  \end{center}
  % \vskip -0.1in
\end{table}
%=================================================

GCBM-E  has a longer training time than GCBM due to differences in the computational complexities of their graph and concept embedding models.
GCBM derives graph and concept vectors directly from frequency statistics without additional training overhead, incurring only $O(n \cdot M \cdot |\mathcal{V}^{(k)}|)$ complexity for concept label calculation, where $n$ is the number of graphs, $M$ is the number of selected concepts, and $|\mathcal{V}^{(k)}|$ is the size of the graph concept vocabulary.
In contrast, GCBM-E  requires training both GIN as the backbone and a Transformer as the embedding model, leading to $O(n \cdot M \cdot d)$ complexity for concept label calculation plus extra overhead from Transformer training: $O(T \cdot n \cdot (L \cdot S^2 \cdot d + L \cdot S \cdot d^2))$ and $O(T \cdot M \cdot (L \cdot S^2 \cdot d + L \cdot S \cdot d^2))$. Here, $T$ is the Transformer’s training epochs, $L$ is its number of network layers, $S$ is the length of the ``graph sentences'', and $d$ is the hidden layer dimension. $S^2$ and $d^2$ terms are the key drivers of GCBM-E ’s higher complexity.
Other components share roughly the same complexity for both models: WL-subtree extraction is $O(n \cdot K \cdot |\mathcal{V}|)$, where $K$ is WL-subtrees height and $|\mathcal{V}|$ is the number of nodes, information gain computation is approximately $O(n \cdot |\mathcal{V}^{(k)}|)$, and the forward computation complexity of the GIN backbone is $O(n \cdot K \cdot |\mathcal{V}|)$.

Overall, while GCBM exhibits slightly lower classification accuracy than GCBM-E , it offers improved training efficiency, serving as a more practical solution for fast graph classification tasks. For GCBM-E , despite the increased complexity introduced by the Transformer, it achieves enhanced accuracy via a higher-quality concept bottleneck.

\section{Theoretical Superiority of GCBMs' AUC Score in Class-Imbalanced Graph Classification}
\label{Appendix_auc}

\subsection{Notations and Premises}
\label{subsec:notations_premises}

\subsubsection{Dataset Definition}
Consider a class-imbalanced binary classification dataset $\mathcal{D} = \{(G_i, \mathbf{c}_i, \mathbf{y}_i)\}_{i=1}^n$, where $\mathbf{h}_i \in \mathbb{R}^d$ and $\mathbf{c}_i \in \mathbb{R}^M$ are the feature vector and concept label of an input instance, respectively. $\mathbf{y}_i \in \{0, 1\}$ denotes the class label.
For the class-imbalanced dataset, $\mathbf{y}_i=0$ and $\mathbf{y}_i=1$ indicate the majority and minority classes, respectively. The prior distribution satisfies:
\begin{equation}
p_1=p(\mathbf{y}_i=1) \ll p(\mathbf{y}_i=0)=p_0,
\end{equation}
where $p_0+p_1=1$ and $p_1\ll 0.5$.

\subsubsection{Assumptions}
\begin{assumption}[Concept Discriminability Assumption]
\label{ass:concept_disc}
The concept set $\mathcal{V}=\{T_1,\ldots,T_M\}$ is highly discriminative. For any concept $T_m \in \mathcal{V}$, there exists a discriminability threshold $\tau > 0$ such that:
\begin{equation}
\text{KL}(p(\mathbf{c}_i=T_m|\mathbf{y}_i=1) \parallel p(\mathbf{c}_i=T_m|\mathbf{y}_i=0)) > \tau,
\end{equation}
where $\text{KL}(\cdot \parallel \cdot)$ denotes the Kullback-Leibler (KL) divergence between two distributions. This implies that the distribution of the same concept differs significantly between the two classes, enabling distinction between the majority and minority classes.
\end{assumption}

\begin{assumption}[Model Prediction Score Distribution Assumption]
\label{ass:score_dist}
The prediction scores of classification models are approximately normally distributed, with no systematic variance differences between the two classes. For simplicity, we assume:
\begin{equation}
\sigma_1^2 \approx \sigma_0^2 = \sigma^2,
\end{equation}
where $\sigma_1^2$ and $\sigma_0^2$ are the variances of prediction scores for minority and majority classes, respectively. This assumption does not affect the validity of core conclusions. For a valid model, the means of the prediction scores for the two classes should satisfy:
\begin{equation}
\mu_1 > \mu_0,
\end{equation}
and $\delta = \mu_1 - \mu_0 > 0$ denotes the difference between the two means.
\end{assumption}

\begin{assumption}[Prediction Independence Assumption]
\label{ass:pred_indep}
Since the prediction process of models for different samples is mutually independent, the distributions of prediction scores for minority and majority class samples satisfy:
\begin{equation}
p(\hat{\mathbf{y}}|\mathbf{y}_i=1) \perp\!\!\!\perp p(\hat{\mathbf{y}}|\mathbf{y}_i=0),
\end{equation}
which means that the prediction score distributions of minority and majority classes are mutually independent.
\end{assumption}

\subsection{Theoretical Proof}
\label{subsec:discriminability_analysis}

\subsubsection{Discriminability Analysis of GCBMs}
A GCBM consists of a backbone and a classifier. The backbone $\Phi: \mathbb{R}^d \to \mathbb{R}^M$ outputs the predicted concepts $\hat{\mathbf{c}} = \Phi(\mathbf{h})$, and the classifier $f: \mathbb{R}^M \to \{0,1\}$ outputs the prediction class $\hat{\mathbf{y}}_c = f(\mathbf{\hat{c}})$.
The model's learning objective is to minimize the weighted risk of classification loss and concept loss:
\begin{equation}
\mathcal{R}_c = {\lambda}_c \mathbb{E}_{\mathbf{h},\mathbf{c}}[\mathcal{L}_{c}(\mathbf{c}, \Phi(\mathbf{h}))] + \mathbb{E}_{\mathbf{h},\mathbf{y}}[\mathcal{L}_{y}(\mathbf{y}, f(\Phi(\mathbf{h})))],
\end{equation}
where $\mathcal{L}_c(\cdot, \cdot)$ is the mean squared error (MSE) loss for concept alignment, $\mathcal{L}_y(\cdot, \cdot)$ is the cross-entropy (CE) loss for classification, and ${\lambda}_c > 0$ is a trade-off parameter that ensures alignment between predicted and true concepts.

A well-trained GCBM satisfies two key constraints, formalized as the following lemmas:
\begin{lemma}[Strong Concept Alignment Constraint]
\label{lem:concept_align}
After optimization with the concept loss $\mathcal{L}_{c}(\mathbf{c}, \Phi(\mathbf{h}))$, the mutual KL divergence between predicted concepts $\hat{\mathbf{c}}$ and concept labels $\mathbf{c}$ is bounded:
\begin{equation}
\begin{split}
\text{KL}(p(\mathbf{c}=T_m|\mathbf{y}=k) \parallel p(\hat{\mathbf{c}}=T_m|\mathbf{y}=k)) &\leq \epsilon, \\
\text{KL}(p(\hat{\mathbf{c}}=T_m|\mathbf{y}=k) \parallel p(\mathbf{c}=T_m|\mathbf{y}=k)) &\leq \epsilon,
\end{split}
\end{equation}
where $T_m \in \mathcal{V}$ and $k \in \{0,1\}$. For simplicity, $T_m$ is omitted in subsequent derivations.
$\epsilon \ll \tau$ is a small alignment error, ensuring the predicted concepts $\hat{\mathbf{c}}$ are highly consistent with the concept labels $\mathbf{c}$.
\end{lemma}

\begin{lemma}[Discriminability Preservation Constraint]
\label{lem:disc_preserve}
The classifier $f:\mathbb{R}^M \to \{0,1\}$ (a fully connected layer with ReLU activation) maintains the discriminability order: a larger KL divergence between input concept distributions leads to a larger KL divergence between output prediction distributions.
Moreover, the discriminability attenuation is bounded: there exists a small constant $\zeta > 0$ (guaranteed by the model structure and training process) such that:
\begin{equation}
\text{KL}(p(\hat{\mathbf{c}}|\mathbf{y}=1) \parallel p(\hat{\mathbf{c}}|\mathbf{y}=0)) - \text{KL}(p(\hat{\mathbf{y}}_c|\mathbf{y}=1) \parallel p(\hat{\mathbf{y}}_c|\mathbf{y}=0)) \leq \zeta,
\end{equation}
where $\zeta \ll \tau$ denotes the maximum allowable discriminability attenuation.
\end{lemma}

For the conditional distributions of the predicted concept $\hat{\mathbf{c}}$ across the two classes, we decompose the logarithmic ratio term into a concept label term and two alignment error terms as follows:
\begin{equation}
\begin{split}
\log \frac{p(\hat{\mathbf{c}}|\mathbf{y}=1)}{p(\hat{\mathbf{c}}|\mathbf{y}=0)} = \log \frac{p(\mathbf{c}|\mathbf{y}=1)}{p(\mathbf{c}|\mathbf{y}=0)} - \log \frac{p(\mathbf{c}|\mathbf{y}=1)}{p(\hat{\mathbf{c}}|\mathbf{y}=1)} 
 - \log \frac{p(\hat{\mathbf{c}}|\mathbf{y}=0)}{p(\mathbf{c}|\mathbf{y}=0)}.
\end{split}
\end{equation}
We then take the expectation of both sides with respect to $\mathbf{c} \sim p(\mathbf{c}|\mathbf{y}=1)$:
\begin{equation}
\label{eq_22}
\begin{split}
\mathbb{E}_{\mathbf{c}|\mathbf{y}=1}\left[\log \frac{p(\hat{\mathbf{c}}|\mathbf{y}=1)}{p(\hat{\mathbf{c}}|\mathbf{y}=0)}\right] = \mathbb{E}_{\mathbf{c}|\mathbf{y}=1}\left[\log \frac{p(\mathbf{c}|\mathbf{y}=1)}{p(\mathbf{c}|\mathbf{y}=0)}\right]- \mathbb{E}_{\mathbf{c}|\mathbf{y}=1}\left[\log \frac{p(\mathbf{c}|\mathbf{y}=1)}{p(\hat{\mathbf{c}}|\mathbf{y}=1)}\right] - \mathbb{E}_{\mathbf{c}|\mathbf{y}=1}\left[\log \frac{p(\hat{\mathbf{c}}|\mathbf{y}=0)}{p(\mathbf{c}|\mathbf{y}=0)}\right].
\end{split}
\end{equation}

By \cref{lem:concept_align}, the predicted concept $\hat{\mathbf{c}}$ is highly consistent with the concept label $\mathbf{c}$, i.e., $\hat{\mathbf{c}} \approx \mathbf{c}$, leading to $\log \frac{p(\hat{\mathbf{c}}|\mathbf{y}=0)}{p(\mathbf{c}|\mathbf{y}=0)} \approx \log \frac{p(\hat{\mathbf{c}}|\mathbf{y}=1)}{p(\mathbf{c}|\mathbf{y}=1)} \approx 0$. Based on this consistency and the definition of KL divergence, we simplify sequentially:
\begin{equation}
\label{eq_23}
\begin{split}
&\mathbb{E}_{\mathbf{c}|\mathbf{y}=1}\left[\log \frac{p(\hat{\mathbf{c}}|\mathbf{y}=1)}{p(\hat{\mathbf{c}}|\mathbf{y}=0)}\right] \approx \mathbb{E}_{\hat{\mathbf{c}}|\mathbf{y}=1}\left[\log \frac{p(\hat{\mathbf{c}}|\mathbf{y}=1)}{p(\hat{\mathbf{c}}|\mathbf{y}=0)}\right] \\
&\quad = \text{KL}(p(\hat{\mathbf{c}}|\mathbf{y}=1) \parallel p(\hat{\mathbf{c}}|\mathbf{y}=0)).
\end{split}
\end{equation}

\begin{equation}
\label{eq_24}
\begin{split}
&\mathbb{E}_{\mathbf{c}|\mathbf{y}=1}\left[\log \frac{p(\hat{\mathbf{c}}|\mathbf{y}=0)}{p(\mathbf{c}|\mathbf{y}=0)}\right] \approx \mathbb{E}_{\mathbf{c}|\mathbf{y}=1}\left[\log \frac{p(\hat{\mathbf{c}}|\mathbf{y}=1)}{p(\mathbf{c}|\mathbf{y}=1)}\right] \\
&\quad \approx \mathbb{E}_{\hat{\mathbf{c}}|\mathbf{y}=1}\left[\log \frac{p(\hat{\mathbf{c}}|\mathbf{y}=1)}{p(\mathbf{c}|\mathbf{y}=1)}\right] = \text{KL}(p(\hat{\mathbf{c}}|\mathbf{y}=1) \parallel p(\mathbf{c}|\mathbf{y}=1)).
\end{split}
\end{equation}

Combining \cref{eq_22,eq_23,eq_24}, we obtain:
\begin{equation}
\begin{split}
&\text{KL}(p(\hat{\mathbf{c}}|\mathbf{y}=1) \parallel p(\hat{\mathbf{c}}|\mathbf{y}=0)) \approx \underbrace{\text{KL}(p(\mathbf{c}|\mathbf{y}=1) \parallel p(\mathbf{c}|\mathbf{y}=0))}_{\text{KL}_c} \\
&\quad - \underbrace{\text{KL}(p(\mathbf{c}|\mathbf{y}=1) \parallel p(\hat{\mathbf{c}}|\mathbf{y}=1))}_{\text{KL}_{err,1}} - \underbrace{\text{KL}(p(\hat{\mathbf{c}}|\mathbf{y}=1) \parallel p(\mathbf{c}|\mathbf{y}=1))}_{\text{KL}_{err,2}}.
\end{split}
\end{equation}

Given that $\text{KL}_{err,1}, \text{KL}_{err,2} \leq \epsilon$ and $\text{KL}_c > \tau$ (\cref{ass:concept_disc}), we derive:
\begin{equation}
\text{KL}(p(\hat{\mathbf{c}}|\mathbf{y}=1) \parallel p(\hat{\mathbf{c}}|\mathbf{y}=0)) \geq \tau - 2\epsilon.
\end{equation}

Let $\tau_1 = \tau - 2\epsilon$, which gives:
\begin{equation}
\text{KL}(p(\hat{\mathbf{c}}|\mathbf{y}=1) \parallel p(\hat{\mathbf{c}}|\mathbf{y}=0)) \geq \tau_1.
\end{equation}

This result shows the predicted concepts inherit the strong discriminability of concept labels. The data processing inequality (DPI) for KL divergence applies to the Markov chain $\mathbf{h} \to \hat{\mathbf{c}} \to \hat{\mathbf{y}}$ of GCBM's inference pipeline, leading to:
\begin{equation}
\text{KL}(p(\hat{\mathbf{y}}_c|\mathbf{y}=1) \parallel p(\hat{\mathbf{y}}_c|\mathbf{y}=0)) \leq \text{KL}(p(\hat{\mathbf{c}}|\mathbf{y}=1) \parallel p(\hat{\mathbf{c}}|\mathbf{y}=0)).
\end{equation}

Combined with \cref{lem:disc_preserve}, the discriminability attenuation of the classifier is bounded, so we derive:
\begin{equation}
\begin{split}
&\text{KL}(p(\hat{\mathbf{y}}_c|\mathbf{y}=1) \parallel p(\hat{\mathbf{y}}_c|\mathbf{y}=0)) \geq \text{KL}(p(\hat{\mathbf{c}}|\mathbf{y}=1) \parallel p(\hat{\mathbf{c}}|\mathbf{y}=0)) - \zeta.
\end{split}
\end{equation}

Let $\tau' = \tau - 2\epsilon - \zeta$. For a well-trained GCBM, $\epsilon, \zeta \ll \tau$ (thus negligible), yielding a positive discriminability threshold $\tau' > 0$. This leads to the following proposition:
\begin{proposition}[GCBM Discriminability Lower Bound]
\label{prop:gcbm_disc}
The KL divergence of GCBM's prediction score distributions between two classes satisfies a strict positive lower bound:
\begin{equation}
\text{KL}_c = \text{KL}(p(\hat{\mathbf{y}}_c|\mathbf{y}=1) \parallel p(\hat{\mathbf{y}}_c|\mathbf{y}=0)) \geq \tau'.
\end{equation}
\end{proposition}

\subsubsection{Discriminability of Black-Box Models}
The black-box model $h: \mathbb{R}^d \to \{0,1\}$ outputs the prediction class $\hat{\mathbf{y}}_b = g(\mathbf{h})$, with the learning objective of minimizing the overall expected risk:
\begin{equation}
\mathcal{R}_b = \mathbb{E}_{\mathbf{h},\mathbf{y}}[\mathcal{L}(\mathbf{y}, g(\mathbf{h}))],
\end{equation}
where $\mathcal{L}(\cdot, \cdot)$ is the cross-entropy loss.

\begin{proposition}[Black-Box Model Discriminability Degradation]
\label{prop:bb_disc_degrad}
Given the prior distribution $p_1 \ll p_0$ (\cref{Appendix_auc}), the black-box model's prediction score distributions for the two classes are nearly indistinguishable, satisfying:
\begin{equation}
p(g(\mathbf{h})|\mathbf{y}=0) \approx p(g(\mathbf{h})|\mathbf{y}=1).
\end{equation}
\end{proposition}
\begin{proof}
The expected risk can be decomposed into class-wise risks:
\begin{equation}
\begin{split}
\mathcal{R}_b &= p_0 \mathbb{E}_{\mathbf{h}|\mathbf{y}=0}[\mathcal{L}(\mathbf{y}, g(\mathbf{h}))] + p_1 \mathbb{E}_{\mathbf{h}|\mathbf{y}=1}[\mathcal{L}(\mathbf{y}, g(\mathbf{h}))].
\end{split}
\end{equation}
Since $p_1 \ll p_0$, the risk is dominated by the majority class, leading to the approximation:
\begin{equation}
\mathcal{R}_b \approx p_0 \mathbb{E}_{\mathbf{h}|\mathbf{y}=0}[\mathcal{L}(\mathbf{y}, g(\mathbf{h}))].
\end{equation}
To minimize risk, the model learns majority class redundant features instead of minority class discriminative features, so the marginal prediction score distribution satisfies:
\begin{equation}
\label{eq_33}
\begin{split}
p(g(\mathbf{h})) = p_0 p(g(\mathbf{h})|\mathbf{y}=0) + p_1 p(g(\mathbf{h})|\mathbf{y}=1) \approx p_0 p(g(\mathbf{h})|\mathbf{y}=0).
\end{split}
\end{equation}
Dividing both sides by $p(g(\mathbf{h})|\mathbf{y}=0) > 0$ gives $\frac{p(g(\mathbf{h}))}{p(g(\mathbf{h})|\mathbf{y}=0)} \approx p_0 \approx 1$, so $p(g(\mathbf{h})) \approx p(g(\mathbf{h})|\mathbf{y}=0)$. Substituting back into \cref{eq_33}:
\begin{equation}
\label{eq_35}
p(g(\mathbf{h})|\mathbf{y}=0) \approx p_0 p(g(\mathbf{h})|\mathbf{y}=0) + p_1 p(g(\mathbf{h})|\mathbf{y}=1).
\end{equation}
Subtracting $p_0 p(g(\mathbf{h})|\mathbf{y}=0)$ from both sides and using $1-p_0=p_1$:
\begin{equation}
\label{eq_36}
p_1 p(g(\mathbf{h})|\mathbf{y}=0) \approx p_1 p(g(\mathbf{h})|\mathbf{y}=1).
\end{equation}
Dividing both sides by $p_1 > 0$ yields the conclusion.
\end{proof}

Using KL divergence to measure distribution differences, the discriminability of the black-box model satisfies:
\begin{equation}
0 < \text{KL}_b = \text{KL}(p(\hat{\mathbf{y}}_b|\mathbf{y}=1) \parallel p(\hat{\mathbf{y}}_b|\mathbf{y}=0)) \ll \tau',
\end{equation}
where $\tau'$ is the GCBM discriminability threshold (\cref{prop:gcbm_disc}). Combining \cref{prop:gcbm_disc,prop:bb_disc_degrad}, we directly get $\text{KL}_c > \text{KL}_b$, meaning GCBM has a larger inter-class prediction score distribution divergence in class-imbalanced scenarios.

\subsubsection{Mean Difference Comparison Between GCBMs and Black-Box Models}
We first establish the relationship between KL divergence and the mean difference of normal distributions via the following lemma:
\begin{lemma}[KL Divergence for Normal Distributions with Equal Variance]
\label{lem:kl_normal}
If two random variables follow normal distributions $P \sim \mathcal{N}(\mu_1, \sigma^2)$ and $Q \sim \mathcal{N}(\mu_0, \sigma^2)$ (equal variance), the KL divergence between $P$ and $Q$ is:
\begin{equation}
\text{KL}(P \parallel Q) = \frac{(\mu_1 - \mu_0)^2}{2\sigma^2}.
\end{equation}
\end{lemma}
\begin{proof}
By the definition of KL divergence $\text{KL}(P \parallel Q) = \mathbb{E}_P\left[\log \frac{p(x)}{q(x)}\right]$, substitute the normal distribution probability density function:
$$p(x) = \frac{1}{\sqrt{2\pi}\sigma}\exp\left(-\frac{(x-\mu_1)^2}{2\sigma^2}\right), \quad q(x) = \frac{1}{\sqrt{2\pi}\sigma}\exp\left(-\frac{(x-\mu_0)^2}{2\sigma^2}\right).$$
Taking the logarithm and simplifying:
$$\log \frac{p(x)}{q(x)} = \frac{(x-\mu_0)^2 - (x-\mu_1)^2}{2\sigma^2} = \frac{2x(\mu_1-\mu_0) + \mu_0^2 - \mu_1^2}{2\sigma^2}.$$
Taking the expectation with respect to $P$ and using $\mathbb{E}_P[x]=\mu_1$, $\mathbb{E}_P[(x-\mu_1)^2]=\sigma^2$:
$$\mathbb{E}_P\left[\log \frac{p(x)}{q(x)}\right] = \frac{2\mu_1(\mu_1-\mu_0) + \mu_0^2 - \mu_1^2}{2\sigma^2} = \frac{(\mu_1-\mu_0)^2}{2\sigma^2}.$$
\end{proof}

For GCBMs, let $\mu_{c1}, \mu_{c0}$ denote the mean prediction scores of minority and majority classes, with mean difference $\delta_c = \mu_{c1} - \mu_{c0}$. By \cref{lem:kl_normal} and \cref{ass:score_dist} ($\sigma_c^2 = \sigma^2$):
\begin{equation}
\text{KL}_c = \frac{\delta_c^2}{2\sigma^2}.
\end{equation}
Similarly, for black-box models with mean difference $\delta_b = \mu_{b1} - \mu_{b0}$:
\begin{equation}
\text{KL}_b = \frac{\delta_b^2}{2\sigma^2}.
\end{equation}

Combining $\text{KL}_c > \text{KL}_b$ (\cref{prop:gcbm_disc,prop:bb_disc_degrad}) and the positive sign of mean differences ($\delta_c, \delta_b > 0$ for valid models), we obtain the following proposition:
\begin{proposition}[GCBM Mean Difference Superiority]
\label{prop:gcbm_mean}
The mean difference of GCBM's prediction scores between two classes is strictly larger than that of black-box models:
\begin{equation}
\delta_c > \delta_b.
\end{equation}
\end{proposition}

\subsubsection{AUC Score Comparison}
AUC (area under the receiver operating characteristic curve) is a key evaluation metric for class-imbalanced binary classification, with the equivalent definition:
\begin{equation}
\text{AUC} = \mathbb{E}_{\mathbf{h}_1 \sim p(\mathbf{h}|\mathbf{y}=1), \mathbf{h}_0 \sim p(\mathbf{h}|\mathbf{y}=0)}[\mathbb{I}(\hat{\mathbf{y}}(\mathbf{h}_1) > \hat{\mathbf{y}}(\mathbf{h}_0))],
\end{equation}
where $\mathbb{I}(\cdot)$ is the indicator function, representing the probability that a random minority sample has a higher prediction score than a random majority sample.

We now prove the core theoretical result—the AUC superiority of GCBMs—via the following theorem:
\begin{theorem}[GCBM AUC Superiority in Class-Imbalanced Graph Classification]
\label{thm:gcbm_auc}
For class-imbalanced binary graph classification datasets, the AUC score of GCBMs is strictly larger than that of black-box models, i.e., $\text{AUC}_c > \text{AUC}_b$.
\end{theorem}
\begin{proof}
By \cref{ass:pred_indep}, the prediction scores of minority and majority samples are independent. Define the difference random variable for GCBM: $Z_c = \hat{\mathbf{y}}_c(\mathbf{h}_1) - \hat{\mathbf{y}}_c(\mathbf{h}_0)$. By \cref{ass:score_dist} (normal distribution) and independent normal distribution properties:
\begin{equation}
Z_c \sim \mathcal{N}(\delta_c, 2\sigma^2).
\end{equation}
The AUC of GCBM is the probability that $Z_c > 0$:
\begin{equation}
\text{AUC}_c = P(Z_c > 0).
\end{equation}
Standardize $Z_c$ to the standard normal variable $Z_c^* = \frac{Z_c - \delta_c}{\sqrt{2}\sigma} \sim \mathcal{N}(0,1)$:
\begin{equation}
\text{AUC}_c = P\left(Z_c^* > -\frac{\delta_c}{\sqrt{2}\sigma}\right).
\end{equation}
By the symmetry of the standard normal distribution ($P(Z^* > -a) = \Phi(a)$, where $\Phi(\cdot)$ is the cumulative distribution function), we get:
\begin{equation}
\text{AUC}_c = \Phi\left(\frac{\delta_c}{\sqrt{2}\sigma}\right).
\end{equation}
Similarly, the AUC of black-box models is:
\begin{equation}
\text{AUC}_b = \Phi\left(\frac{\delta_b}{\sqrt{2}\sigma}\right).
\end{equation}
Since $\Phi(\cdot)$ is a strictly increasing function and $\delta_c > \delta_b$ (\cref{prop:gcbm_mean}), we conclude $\text{AUC}_c > \text{AUC}_b$.
\end{proof}

\begin{remark}
\label{rem:gcbm_auc_intuition}
Black-box models tend to be dominated by redundant features of the majority class (frequently appearing but minority-irrelevant features), weakening their minority class classification capability. In contrast, GCBMs are forced to learn cross-class discriminative features through concept alignment (\cref{lem:concept_align}), which are mapped to the concept space and associated with all classes, significantly enhancing minority class identification. This conclusion can be extended to multi-class class-imbalanced settings with consistent validity.
\end{remark}

\section{Quantitative Assessment of Interpretability}
\label{Appendix_M_I}
GCBMs extract key subgraphs without additional training, which is simply achieved by adjusting the number of selected concepts $M_I$. Table~\ref{tab:M_I} reports the AUC scores of GCBMs under different $M_I$ values.
\begin{table}[!htb]
  \caption{AUC scores (mean\%) of GCBM/GCBM-E with different $M_I$ values.}
  \label{tab:M_I}
  \begin{center}
    \begin{small}
      % \begin{sc}
        \setlength{\tabcolsep}{4pt}
        \begin{tabular}{lcc}
          \toprule
          Dataset   & Solubility            & Benzene   \\
          \midrule
          $M_I$=2       & \textbf{91.2 }        & 83.6      \\
          $M_I$=5       & \underline{81.2}      & \underline{85.9}      \\
          $M_I$=10      & 72.8                  & \textbf{86.1 }     \\
          $M_I$=15      & 68.9                  & 84.0      \\
          \bottomrule
        \end{tabular}
      % \end{sc}
    \end{small}
  \end{center}
  % \vskip -0.1in
\end{table}

\section{Detailed Procedures for Intervention Experiments}
\label{Appendix_intervention}

\subsection{Definition}

For graph $G_i$, $\hat{\mathbf{y}}_i$ and $\mathbf{y}_i$ denote the predicted label and the ground-truth label, respectively. $\hat{\mathbf{c}}_i$ and $\mathbf{c}_i$ represent the predicted concept label and the graph-truth concept label, respectively.
$W_F$ is the weight matrix of the final classification layer, where $W_{F[i, j]}$ denotes the weight of the $j$-th concept corresponding to the $i$-th class in $W_F$, and $\text{sim}(\cdot, \cdot)$ denotes the cosine similarity between two vectors.

\subsection{Sample Selection Criteria}

Intervention experiments are conducted exclusively on misclassified graphs satisfying two criteria: 
\begin{itemize}
\item High consistency between the predicted concept label and the ground-truth concept label: $\text{sim}(\hat{\mathbf{c}}_i, \mathbf{c}_i) \geq \tau_c$, where $\tau_c$ is a similarity threshold ensuring valid concept activation;

\item The graph belongs to the target error type: $\mathbf{y}_i=\text{cls 1}$ but $\hat{\mathbf{y}}_i=\text{cls 0}$.
\end{itemize}
The intervention graph set is defined as $S = \{G_i \mid \hat{\mathbf{y}}_i=\text{cls 0}, \mathbf{y}_i=\text{cls 1}, \text{sim}(\hat{\mathbf{c}}_i, \mathbf{c}_i) \geq \tau_c\}$.

\subsection{Core Parameter Calculation for Intervention}

For the intervention graph set $S$, we first compute two key statistics reflecting the overall error properties: the average logit difference $\overline{\Delta a}$ and the average activation value $\overline{f^{(\text{cpt})}}$ of the target concept.

\subsubsection{Average Logit Difference $\overline{\Delta a}$}

For graph $G_i \in S$, the pre-intervention logit difference is defined as:
\begin{equation}
\Delta a_{G_i} = \text{logit}_{G_i}(\text{cls 0}) - \text{logit}_{G_i}(\text{cls 1}),
\end{equation}
where $\text{logit}_{G_i}(\text{cls}) = W_{F[\text{cls}, :]} \cdot \hat{\mathbf{c}}_{i}^T + b_F$ and $b_F$ is the bias term of the final classification layer. The average logit difference of $S$ is:
\begin{equation}
\overline{\Delta a} = \frac{1}{|S|} \sum_{G_i \in S} \Delta a_{G_i}.
\end{equation}
$\overline{\Delta a} > 0$ indicates that the model's logit value for the incorrect class ($\text{cls 0}$) is higher than that for the correct class ($\text{cls 1}$) on $S$, and the intervention aims to reverse this gap.

\subsubsection{Average Activation Value of Target Concept $\overline{f^{(\mathrm{cpt})}}$}

For the target concept (cpt $j$), the activation value of $G_i$ is $\hat{\mathbf{c}}_i[j]$, corresponding to the $j$-th entry of $\hat{\mathbf{c}}_i$. 
The average activation value of the target concept on $S$ is defined as:
\begin{equation}
\overline{f^{(\text{cpt j})}} = \frac{1}{|S|} \sum_{G_i \in S} \hat{\mathbf{c}}_i[j].
\end{equation}
This parameter quantifies the average relevance of the target concept to $S$.

\subsection{Weight Adjustment Rules}

The objective of the intervention is to adjust the target concept’s weights for the two classes, enhancing its contribution to the correct class while suppressing that to the incorrect class.

\subsubsection{Weight Adjustment Value $\Delta w$}

The adjustment value is calculated as:
\begin{equation}
\Delta w = \frac{\overline{\Delta a} + b}{2 \cdot \overline{f^{(\text{cpt})}}},
\end{equation}
where $b$ is the a margin aiming to ensure the correct class’s average logit stably exceeds the incorrect class’s.

\subsubsection{Weight Modification Operations}

For each target concept (cpt $j$), weight adjustments follow:
\begin{itemize}
\item Increase the weight for the correct class (cls 1): $W_{F[\text{cls 1}, \text{cpt j}]} \leftarrow W_{F[\text{cls 1}, \text{cpt j}]} + \Delta w$;

\item Decrease the weight for the incorrect class (cls 0): $W_{F[\text{cls 1}, \text{cpt j}]} \leftarrow W_{F[\text{cls 1}, \text{cpt j}]} - \Delta w$;

\item Retain all other weights unchanged to avoid interfering with the classification logic of non-target concepts.
\end{itemize}

\subsection{Intervention Implementation Details}

We perform two rounds of interventions on GCBM-E, focusing on key concepts 20 and 42 (verified as key substructures via human evaluation in interpretability experiments).
Detailed parameters are summarized in Table~\ref{tab:Intervention_parameters}.
%========================================================================================
\begin{table}[!htb]
  \caption{Detailed parameters of the intervention experiments.}
  \label{tab:Intervention_parameters}
  \begin{center}
    \begin{small}
  %     \begin{sc}
        \setlength{\tabcolsep}{2.5pt} 
        \begin{tabular}{lccccc cc cc} 
          \toprule
          Concept  & $\overline{\Delta a}$ & $\overline{f^{(j)}}$ & $\Delta w$ & $W_{F[\text{cls 0}, \text{cpt j}]}$ & $W_{F[\text{cls 1}, \text{cpt j}]}$ \\
          \midrule
          cpt 20  & 0.51 & 0.8426  & 0.4213 & 0.3842 $\to$ -0.0371 & -0.0112 $\to$ 0.4101  \\
          cpt 42  & 0.32 & 0.7635  & 0.3405 & 0.1876 $\to$ -0.2466 & 0.0939 $\to$ 0.5281  \\
          \bottomrule
        \end{tabular}
  %     \end{sc}
    \end{small}
  \end{center}
  % \vskip-0.1in
\end{table}

Note that the concept similarity threshold $\tau_c$ and the intervention margin $b$ are tunable hyperparameters. Adjusting them can further optimize the intervention effects. The default settings ($\tau_c=0.6$, $b=0.2$) determined via 5-fold cross-validation to balance concept activation validity and intervention stability. After each round of weight modification, the model is re-evaluated on the entire test set to ensure the number of new misclassifications is fewer than that of corrected samples.

\end{document}